\documentclass{article}

\usepackage[final]{neurips_2023}

\usepackage[utf8]{inputenc} 
\usepackage[T1]{fontenc}    
\usepackage{hyperref}       
\usepackage{url}            
\usepackage{booktabs}       
\usepackage{amsfonts}       
\usepackage{nicefrac}       
\usepackage{microtype}      
\usepackage{xcolor}         
\usepackage{wrapfig}

\usepackage{graphicx}
\usepackage{subfigure}

\usepackage{amsmath}
\usepackage{amssymb}
\usepackage{mathtools}
\usepackage{amsthm}

\usepackage{enumitem}
\usepackage{multirow}
\usepackage{booktabs} 
\usepackage{amsfonts}
\usepackage{bm}
\usepackage{tikz}
\usepackage{amssymb}
\usepackage{xcolor}
\usepackage{cleveref}
\theoremstyle{plain}
\newtheorem{theorem}{Theorem}[section]

\theoremstyle{definition}
\newtheorem{definition}[theorem]{Definition}
\newtheorem{assumption}[theorem]{Assumption}
\theoremstyle{remark}

\usepackage{amsmath,amsfonts,bm}
\usepackage{xifthen}
\definecolor{lightgray}{gray}{0.9}

\newcommand*\dbar[1]{\overline{\overline{\lower0.2ex\hbox{$#1$}}}}

\newcommand{\harrow}[1]{\mathstrut\mkern2.5mu#1\mkern-11mu\raise1.6ex
\hbox{$\scriptscriptstyle\rightharpoonup$}}

\definecolor{Gray}{gray}{0.9}

\newcommand{\optionalsuperscript}[1]{\ifthenelse{\isempty{#1}}{}{^{(#1)}}}

\newtheorem{property}{Property}

\def\eqref#1{equation~\ref{#1}}

\def\1{\bm{1}}

\def\mE{{\bm{E}}}

\def\mX{{\bm{X}}}

\DeclareMathAlphabet{\mathsfit}{\encodingdefault}{\sfdefault}{m}{sl}
\SetMathAlphabet{\mathsfit}{bold}{\encodingdefault}{\sfdefault}{bx}{n}

\def\gG{{\mathcal{G}}}

\def\gT{{\mathcal{T}}}

\newcommand{\R}{\mathbb{R}}

\title{\textsc{\textsc{HyTrel}}: \underline{Hy}pergraph-enhanced \\ \underline{T}abular Data \underline{Re}presentation \underline{L}earning}

\author{%
  Pei Chen$^{1*}$, Soumajyoti Sarkar$^{2}$, Leonard Lausen$^2$, \\ \textbf{Balasubramaniam Srinivasan$^2$,  Sheng Zha$^2$, Ruihong Huang$^1$, George Karypis$^2$}\\
  $^1$Texas A\&M University,
  $^2$Amazon Web Services \\
  \texttt{\{chenpei,huangrh\}@tamu.edu}, \\
  \texttt{\{soumajs,lausen,srbalasu,zhasheng,gkarypis\}@amazon.com}
}

\begin{document}

\maketitle

\begin{abstract}
Language models pretrained on large collections of tabular data have demonstrated their effectiveness in several downstream tasks.
However, 
many of these models do not take into account the row/column
permutation invariances, hierarchical structure,
etc. that exist in tabular data. 
To alleviate these limitations, we propose \textbf{\textsc{HyTrel}}, a tabular language model, that captures the permutation invariances and three more  \textbf{\textit{structural properties}} of tabular data by using hypergraphs--where the table cells make up the nodes and the cells occurring jointly together in each row, column, and the entire table are used to form three different types of hyperedges. We show that
\textsc{HyTrel} is maximally invariant under certain conditions for tabular data, i.e., two
tables obtain the same representations via \textsc{HyTrel}
\textit{iff} the two tables are identical up to permutations. 
Our empirical results demonstrate that \textsc{HyTrel} \textbf{consistently} outperforms other competitive baselines on four downstream tasks with minimal pretraining, illustrating the advantages of incorporating the inductive biases associated with tabular data into the representations.
%
%
Finally, our qualitative analyses showcase that \textsc{HyTrel} can assimilate the table structures to generate robust representations for the cells, rows, columns, and the entire table. \footnotemark[1]
%
\end{abstract}

\renewcommand{\thefootnote}{\fnsymbol{footnote}}
\footnotetext[1]{Work done as an intern at Amazon Web Services.}
\renewcommand{\thefootnote}{\arabic{footnote}}
\footnotetext[1]{Code is available at: \url{https://github.com/awslabs/hypergraph-tabular-lm}}

\section{Introduction}
\label{intro}


Tabular data that is organized in bi-dimensional matrices are widespread in webpages, documents, and databases. Understanding tables can benefit many tasks such as table type classification, table similarity matching, and knowledge extraction from tables (e.g., column annotations) among others. Inspired by the success of pretrained language models in natural language tasks, recent studies ~\citep{yin20acl,yang-etal-2022-tableformer} proposed Tabular Language Models (TaLMs) that perform pretraining on tables via self-supervision to generate expressive representations of tables for downstream tasks.



Among the TaLMs, many works~\citep{herzig-etal-2020-tapas,yin20acl,deng2020turl,iida-etal-2021-tabbie} serialize tables to a sequence of tokens for leveraging existing pretrained language model checkpoints and textual self-supervised objectives like the Masked Language Modeling. 
However, due to the linearization of tables to strings, these models do not explicitly incorporate
the structural properties of a table, e.g., the  invariances to arbitrary permutations of rows and columns (independently). 
Our work focuses on obtaining representations of tables that take table structures into account.
We hypothesize that incorporating such properties into the table representations will benefit many downstream table understanding tasks.


\begin{figure}[t]
\begin{center}
\includegraphics[width=320pt]{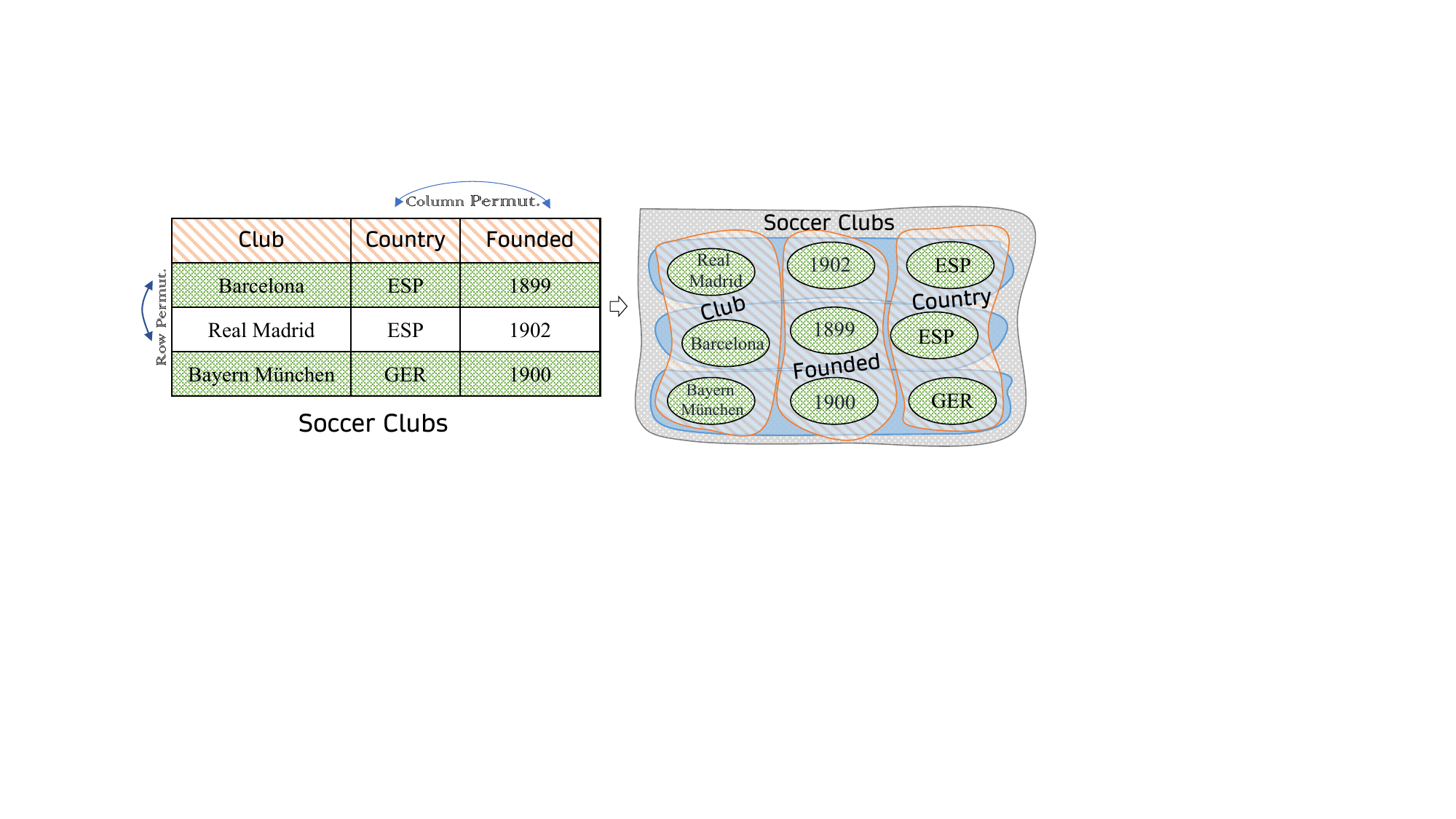}
\caption{\small An example of modeling a table as a hypergraph. Cells make up the nodes and the cells in each row, column, and the entire table form hyperedges. The table caption and the header names are used for the names of the table and column hyperedges. The hypergraph keeps the four structural properties of tables, e.g., the invariance property of the table as the row/column permutations result in the same hypergraph.
}
\label{fig:table2hypergraph}
\end{center}
\end{figure}

{\bf Motivation:} Tabular data is structurally different in comparison to other data modalities such as images, audio, and plain texts. We summarize four \textbf{\textit{structural properties}} present in the tables below:

\begin{itemize}[noitemsep, topsep=5pt, leftmargin=*]
  \item Most tables are invariant to row/column permutations. This means, in general, if we arbitrarily (and independently) permute the rows or columns of a table, it is still an equivalent table. For other tables with an explicit ordering of rows or columns, we can make them permutation invariant by appropriately adding a ranking index as a new column or row.   
  \item Data from a single column are semantically similar -- for example, they oftentimes have the same semantic types. Similarly, the cells within a single row together describe the attributes of a sample within the table and the cells cannot be treated in silos.
  
  \item The interactions within cells/rows/columns are not necessarily pairwise, i.e., the cells within the same row/column, and rows/columns from the same table can have high-order multilateral relations~\citep{chien2022youn}.
  \item Information in tables is generally organized in a hierarchical fashion where the information at the table-level can be aggregated from the column/row-level, and further from the cell-level.
\end{itemize}

However, the linearization-based approaches are not designed to explicitly capture most of the above properties.
We aim to address the limitation by modeling all the aforementioned structural properties as inductive biases while learning the table representations.



{\bf Our Approach: } In line with recent studies~\citep{deng2020turl,yang-etal-2022-tableformer,10.1145/3447548.3467434} which have elucidated upon the importance of the structure of a table, we propose the \textsc{HyTrel} that uses hypergraphs to model the tabular data. 
We propose a modeling paradigm that aims \textit{\textbf{capture all of the four properties} }directly. 
Figure~\ref{fig:table2hypergraph} provides an example of how a hypergraph is constructed from a table.
As observed, converting a table into a hypergraph allows us to incorporate the first two properties inherent to the nature of hypergraphs. Hypergraphs seamlessly allow the model to incorporate row/column permutation invariances, as well as interactions among the cells within the same column or row. 
Moreover, the proposed hypergraph structure can capture the high-order (not just pairwise) interactions for the cells in a column or a row, as well as from the whole table, and an aggregation of hyperedges can also help preserve the hierarchical structure of a table.


{\bf Contributions:} Our theoretical analysis and empirical results demonstrate the advantages of modeling the four structural properties. We first show that \textsc{HyTrel} is maximally invariant when modeling tabular data (under certain conditions), i.e. if two tables get the same representations via the hypergraph table learning function, then the tables differ only by row/column permutation (independently) actions and vice versa.
Empirically, we pretrain \textsc{HyTrel} on publicly available tables using two self-supervised objectives: a table content based ELECTRA\footnotemark[2] objective ~\citep{clark2020electra,iida-etal-2021-tabbie} and a table structure dependent contrastive objective ~\citep{wei2022augmentations}. The evaluation of the pretrained \textsc{HyTrel} model on four downstream tasks (two knowledge extraction tasks, a table type detection task, and a table similarity prediction task) shows that \textsc{HyTrel} can achieve state-of-the-art performance. 

We also provide an extensive qualitative analysis of \textsc{HyTrel}--including visualizations that showcase that (a) \textsc{HyTrel} representations are robust to arbitrary permutations of rows and columns (independently), (b) \textsc{HyTrel} can incorporate the hierarchical table structure into the representations, (c) \textsc{HyTrel} can achieve close to state-of-the-art performance even without pretraining, and the model is extremely efficient with respect to the number epochs for pretraining in comparison to prior works, further demonstrating the advantages of \textsc{HyTrel} in modeling the structural properties of tabular data. 
In Appendix~\ref{appdx:further_analysis}, we provide additional analysis that demonstrates \textsc{HyTrel}'s ability to handle input tables of arbitrary size and underscore the importance of the independent row/column permutations.



\footnotetext[2]{We use the name ELECTRA following ~\citet{iida-etal-2021-tabbie} but it is different from the ELECTRA objective used in language modeling~\citep{clark2020electra}. We do not utilize an auxiliary model to generate corruption.}



\section{\textsc{HyTrel} Model}


Formally, a table in our work is represented as $\mathcal{T} =[M,H,R]$, where $M$ is the caption, 
$H=[h_1, h_2, h_3,...,h_m]$ are the $m$ column headers, 
$R$ represents the $n$ rows $[R_1, R_2,R_3,...,R_n]$. Each row $R_i$ has $m$ cells $[c_{i1}, c_{i2}, c_{i3}, ...,c_{im}]$. The caption, header, and cell values can be regarded as sentences that contain several words. 
We note that each cell $c_{ij}$ or header $h_j$ also belongs to the corresponding column $C_j$. 
We use $C=[C_1, C_2, C_3,...,C_m]$ to represent all the columns that include the headers, so a table can also be defined as $\mathcal{T} = [M,C]$.

\begin{figure*}[ht]
\begin{center}
\includegraphics[width=300pt]{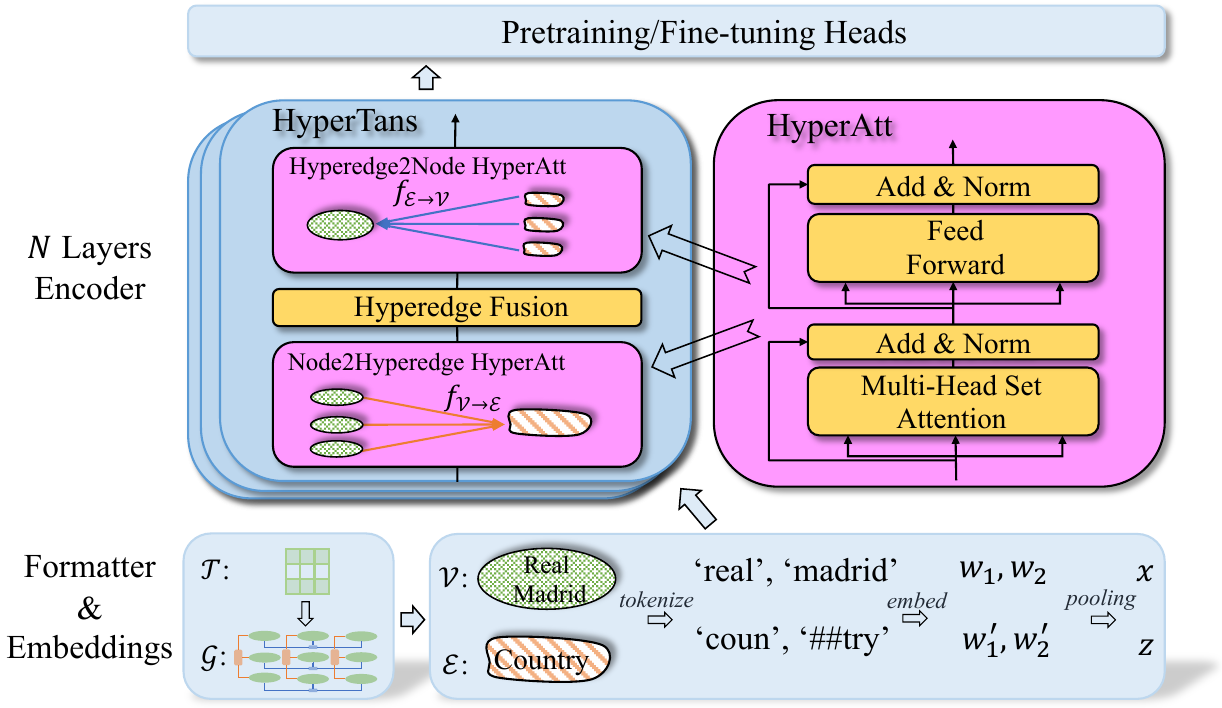}
\caption{\small The overall framework of \textsc{HyTrel}. We first turn a table $\mathcal{T}$ into a hypergraph $\mathcal{G}$ and then initialize the embeddings of the nodes $\mathcal{V}$ and hyperedges $\mathcal{E}$. After that, we encode the hypergraph using stacked multiple-layer hypergraph-structure-aware transformers (\texttt{HyperTrans}). Each \texttt{HyperTrans} layer has two attention blocks that work on hypergraph (\texttt{HyperAtt}) and one \texttt{Hyperedge Fusion} block. Lastly, we use the node and hyperedge representations from the final layer for pretraining and fine-tuning.}
\label{fig:framework}
\end{center}
\end{figure*}

\subsection{Formatter \& Embedding Layer}
\label{subsec:construction}
The formatter transforms a table into a hypergraph. As shown in Figure~\ref{fig:table2hypergraph}, given a table $\mathcal{T}$, we construct a corresponding hypergraph $\mathcal{G}=(\mathcal{V},\mathcal{E})$, where $\mathcal{V}$, $\mathcal{E}$ denote the set of nodes and  hyperedges respectively. 
We treat each cell $c_{ij}$ as a node $v_{ij}\in \mathcal{V}$, and each row $R_i$, each column $C_j$, and the entire table $\mathcal{T}$ as hyperedges $e_i^c, e_j^r, e^t \in \mathcal{E}$ $_{(1\leq i \leq n,1 \leq j \leq m) }$, respectively. As a part of our hypergraph construction, each cell node $v_{ij}$ is connected to 3 hyperedges: its column hyperedge $e_i^c$, row hyperedge $e_j^r$, and the table hyperedge $e^t$. 
The hypergraph can be conveniently be represented as an incidence matrix  $\mathbf{B} \in \{0,1\}^{mn \times (m+n+1)}$, where $\mathbf{B}_{ij} =1$ when node $i$ belong to hyperedge $j$ and $\mathbf{B}_{ij}=0$ otherwise.      

An embedding layer is then employed over the nodes and hyperedges . Each node $v_{ij}$ corresponds to a cell $c_{ij}$ that has several tokens, and we obtain the feature vector $ \mathbf{X}_{v_{ij},:} \in \mathbb{R}^{F}$ for a given node by feeding its constituent cell tokens into the embedding layer and then averaging the embeddings over the tokens. 
After obtaining node embeddings $\mathbf{X} \in \mathbb{R}^{nm \times F}$ for all nodes, a similar transformation is applied over hyperedges. 
For different hyperedge types, we use different table content for their initialization : 
for a column hyperedge $e_i^c$, we use all the tokens from the corresponding header $h_i$. 
For the table hyperedge $e^t$, we use the the entire caption associated with $M$. 
For the row hyperedge $e_j^r$, when no semantic information is available, we randomly initialize them as $\mathbf{{S}}_{e^r_j,:} \in \mathbb{R}^{F}$. 
Performing the above operations yields an initialization for all the hyperedge embeddings $\mathbf{S} \in \mathbb{R}^{(m+n+1) \times F}$.







\subsection{Hypergraph Encoder}

After embedding, we propose to use a structure-aware transformer module (\texttt{HyperTrans}) to encode the hypergraphs. \texttt{HyperTrans} encoder can encode the table content, structure, and relations among the table elements (including cells, headers, captions, etc.). As shown in Figure~\ref{fig:framework}, one layer of the \texttt{HyperTrans} module is composed of two hypergraph attention blocks (\texttt{HyperAtt}, $f$)~\citep{,chien2022you}  that interact with the node and hyperedge representations, and one \texttt{Hyperedge Fusion} block. The first \texttt{HyperAtt} block is the \texttt{Node2Hyperedge} attention block as defined below: 


\begin{equation} 
\label{eqn:start}
\mathbf{\tilde{S}}_{e,:}^{(t+1)}=f_{\mathcal{V} \rightarrow \mathcal{E}}\left(K_{e, \mathbf{X}^{(t)}} \right)
\end{equation}

Where $f_{\mathcal{V} \rightarrow \mathcal{E}}$ is a hypergraph attention function defined from nodes to hyperedges. $K_{e, \mathbf{X}}=\left\{\mathbf{X}_{v,:}: v \in e\right\}$ denotes the sets of hidden node representations included in the hyperedge $e$. The \texttt{Node2Hyperedge} block will aggregate information to hyperedge $e$ from its constituent nodes $v \in e$.

We then use a \texttt{Hyperedge Fusion} module (a Multilayer Perceptron Network, $\operatorname{MLP}$) to propagate the hyperedge information from the last step, as defined below:

\begin{equation} 
\label{eq:2}
\mathbf{S}_{e,:}^{(t+1)}=\operatorname{MLP}\left(\mathbf{S}_{e,:}^{(t)} ;\mathbf{\tilde{S}}_{e,:}^{(t+1)}\right)
\end{equation}

A second \texttt{HyperAtt} block \texttt{Hyperedge2Node} then aggregates information from a hyperedge to its constituent nodes as follows:

\begin{equation}
 \mathbf{X}_{v,:}^{(t+1)}=f_{\mathcal{E} \rightarrow \mathcal{V}}\left(L_{v, \mathbf{S}^{(t+1)}} \right)
\end{equation}

Where $f_{\mathcal{E} \rightarrow \mathcal{V}}$ is another hypergraph attention function defined from hyperedges to nodes. $L_{v, \mathbf{S}}=\left\{\mathbf{S}_{e,:}: v \in e\right\}$ is defined as the sets of hidden representations of hyperedges that contain the node $v$.

As for the \texttt{HyperAtt} block $f$, similar to transformer~\citep{NIPS2017_3f5ee243}, it is composed of one multi-head attention, one Position-wise Feed-Forward Network ($\operatorname{FFN}$), two-layer normalization ($\operatorname{LN}$)~\citep{Ba2016LayerN} and two skip connections~\citep{7780459}, as in Figure~\ref{fig:framework}. However, we do not use the self-attention~\citep{NIPS2017_3f5ee243} mechanism from the transformer model because it is not designed to keep the invariance structure of tables or hypergraphs. Inspired by the deep set models~\citep{NIPS2017_f22e4747,lee2019set}, we use a set attention mechanism that can 
keep the permutation invariance of a table.  We define \texttt{HyperAtt} $f$ as follows:


\begin{equation}
f_{\mathcal{V} \rightarrow \mathcal{E} \; \text{or} \; \mathcal{E} \rightarrow \mathcal{V}}(\mathbf{I}):=\operatorname{L N}(\mathbf{Y}+\operatorname{FFN}(\mathbf{Y})) 
\end{equation}


Where $\mathbf{I}$ is the input node or hyperedge representations. The intermediate representations $\mathbf{Y}$ is obtained by:

\begin{equation}
\label{eqn:end}
\mathbf{Y}=\operatorname{LN}\left(\omega+\operatorname{SetMH}(\omega, \mathbf{I}, \mathbf{I})\right)
\end{equation}

Where $\operatorname{SetMH}$ is the multi-head set attention mechanism defined as: 
\begin{equation}
\operatorname{SetMH}(\omega, \mathbf{I}, \mathbf{I})=\|_{i=1}^h \mathbf{O}_{i}  
\end{equation}

and 
\begin{equation}
\mathbf{O}_{i}=\operatorname{Softmax}\left(\omega_{i}\left( \mathbf{I} \mathbf{W}_{i}^{K} \right)^T\right) \left( \mathbf{I} \mathbf{W}_{i}^{V}\right)
\end{equation}

Where $\omega$ is a learnable weight vector as the query and $\omega := \|_{i=1}^h \omega_{i}$, $\mathbf{W}_{i}^{K}$ and $\mathbf{W}_{i}^{V}$ are the weights for the key and value projections, $\|$ means concatenation.

So the \texttt{HyperTrans} module will update node and hyperedge representations alternatively. This mechanism enforces the table cells to interact with the columns, rows, and the table itself. Similar to BERT$_{base}$~\citep{devlin-etal-2019-bert} and  TaBERT$_{base}$~\citep{yin20acl} , we stack $12$ layers of \texttt{HyperTrans}.



\subsection{Invariances of the \textsc{HyTrel} Model}

Let $\phi:\mathcal{T} \mapsto \mathbf{z} \in \mathbb{R}^d$ be our target function which captures the desired row/column permutation invariances of tables (say for tables of size $n \times m$).
Rather than working on the table $\mathcal{T}$ directly, 
the proposed \textsc{HyTrel} model works on a hypergraph (via Eqns~(\ref{eqn:start}-\ref{eqn:end})) that has an incidence matrix $\mathbf{B}$ of size ${mn \times (m+n+1)}$.
Correspondingly, we shall refer to \textsc{HyTrel} as a function $g:\mathbf{B} \mapsto \mathbf{y} \in \mathbb{R}^k$.

In this section we will make the connections between the properties of the two functions $\phi$ and $g$, demonstrating a maximal invariance between the two--as a result of which we prove that our \textsc{HyTrel} can also preserve the permutation invariances of the tables.
First, we list our assumptions and resultant properties of tabular data. 
Subsequently, we present the maximal invariance property of $\phi$ and our hypergraph-based learning framework $g$.
As a part of our notation, we use $[n]$ to denote $\{1,2,\ldots, n\}$.
Preliminaries and all detailed proofs are presented in the Appendix~\ref{sec:prelims} and~\ref{sec:proofs} respectively.

\begin{assumption}
\label{assumption:septable}
For any table $(\gT_{ij})_{i \in [n], j \in [m]}$ (where $i, j$ are indexes of the rows, columns), an arbitrary group action $a \in \mathbb{S}_n \times \mathbb{S}_m$ acting appropriately on the rows and columns leaves the target random variables associated with tasks on the entire table unchanged.
\end{assumption}

This assumption is valid in most real-world tables -- as reordering the columns and the rows in the table oftentimes doesn't alter the properties associated with the entire table (e.g. name of the table, etc). 
As noted earlier, for tables with an explicit ordering of rows or columns, we can make them permutation invariant by adding a ranking index as a new column or row appropriately.
To model this assumption, we state a property required for functions acting on tables next.


\begin{property}
\label{pro:1}
A function $\phi:\mathcal{T} \mapsto \mathbf{z} \in R^d$ which satisfies Assumption \ref{assumption:septable} and defined over tabular data must be invariant to actions from the (direct) product group $\mathbb{S}_n \times \mathbb{S}_m$ acting appropriately on the table i.e. $\phi(a \cdot \gT) = \phi(\gT)\;\; \forall a \in \mathbb{S}_n \times \mathbb{S}_m$.
\end{property}



However, \textsc{HyTrel} (or the function $g$ via hypergraph modeling) through Eqns~(\ref{eqn:start}-\ref{eqn:end})) models invariances of the associated incidence matrix to the product group $\mathbb{S}_{mn}\times \mathbb{S}_{m+n+1}$ (proof presented in the appendix).
To make the connection between the two, we present the maximal invariance property of our proposed \textsc{HyTrel} model.

\begin{theorem}
\label{theorem}
A  continuous function $\phi:\mathcal{T} \mapsto \mathbf{z} \in \mathbb{R}^d$ over tables is maximally invariant when modeled as a function $g:\mathbf{B} \mapsto \mathbf{y} \in \mathbb{R}^k$ over the incidence matrix of a hypergraph $\gG$ constructed per Section~\ref{subsec:construction} (Where $g$ is defined via Eqns~(\ref{eqn:start}-\ref{eqn:end})) if $\exists$ a bijective map between the space of tables and incidence matrices (defined over appropriate sizes of tables, incidence matrices). 
That is, $\phi(\mathcal{T}_1) = \phi(\mathcal{T}_2)$ iff $\mathcal{T}_2$ is some combination of row and/or column permutation of $\mathcal{T}_1$ and $g(\mathbf{B_1}) = g(\mathbf{B_2})$ where $\mathbf{B}_1, \mathbf{B}_2$ are the corresponding (hypergraph) incidence matrices of tables $\mathcal{T}_1, \mathcal{T}_2$.
\end{theorem}
{\em Proof Sketch}: Detailed proof is provided in \Cref{sec:proofs}. The above theorem uses Lemma 1 from \citep{tyshkevich1996line} and applies the Weisfeiler-Lehman test of isomorphism over the star expansion graphs of the hypergraphs toward proving the same.

As a consequence of Theorem~\ref{theorem}, two tables identical to permutations will obtain the same representation, which has been shown to improve generalization performance \citep{lyle2020benefits}.

\subsection{Pretraining Heads}
\noindent {\bf ELECTRA Head}: In the ELECTRA pretraining setting, we first corrupt a part of the cells and the headers from a table and then predict whether a given cell or header has been corrupted or not~\citet{iida-etal-2021-tabbie}. Cross-entropy loss is used together with the binary classification head.






\noindent {\bf Contrastive Head}: In the contrastive pretraining setting, we randomly corrupt a table-transformed hypergraph by masking a portion of the connections between nodes and hyperedges, as inspired by the hypergraph contrastive learning~\citep{wei2022augmentations}. 
For each hypergraph, we corrupt two augmented views and use them as the positive pair, and use the remaining in-batch pairs as negative pairs. 
Following this, we contrast the table and column representations from the corresponding hyperedges. The InfoNCE~\citep{DBLP:journals/corr/abs-1807-03748} objective is used for optimization as in~\ref{eq:InfoNCE}.


\begin{equation}
\label{eq:InfoNCE}
loss =-\log \frac{\exp \left(\boldsymbol{q} \cdot \boldsymbol{k}_{+} / \tau\right)}{\sum_{i=0}^K \exp \left(\boldsymbol{q} \cdot \boldsymbol{k}_i / \tau\right)}
\end{equation}

where $(\boldsymbol{q}, \boldsymbol{k}_{+})$ is the positive pair, and $\tau$ is a temperature hyperparameter.  

\section{Experiments}
\label{sec:exp}

\subsection{Pre-training} 

\noindent {\bf Data} In line with previous TaLMs~\citep{yin20acl,iida-etal-2021-tabbie}, we use tables from Wikipedia and Common Crawl for pretraining. 
We utilize preprocessing tools provided by~\citet{yin20acl} and collect a total of 27 million tables (1\% are sampled and used for validation).\footnotemark[3]  During pretraining, we truncate large tables and retain a maximum of 30 rows and 20 columns for each table, with a maximum of 64 tokens for captions, column names, and cell values. 
It is important to note that the truncation is solely for efficiency purposes and it does not affect \textsc{HyTrel}'s ability to deal with large tables, as elaborated in appendix~\ref{appdx:tab_size}.

\footnotetext[3]{As the version of Wikipedia used by~\citep{yin20acl} is not available now, we use an updated version so we collect slightly more tables than previous TaLMs.} 

\noindent {\bf Settings} With the ELECTRA pretraining objective, we randomly replace 15\% of the cells or headers of an input table with values that are sampled from all the pretraining tables based on their frequency, as recommended by~\citet{iida-etal-2021-tabbie}. 
With the contrastive pretraining objective, we corrupted 30\% of the connections between nodes and hyperedges for each table to create one augmented view. The temperature $\tau$ is set as 0.007. For both objectives, we pretrain the \textsc{HyTrel} models for 5 epochs. More details can be found the Appendix~\ref{appdx:pretrain}.

\subsection{Fine-tuning\protect\footnotemark[4]}
\footnotetext[4]{More details about experimental settings, the datasets, and the baselines can be found the Appendix~\ref{appdx:fine-tuning}}

After pretraining, we use the \textsc{HyTrel} model as a table encoder to fine-tune downstream table-related tasks. In order to demonstrate that our model does not heavily rely on pretraining or on previous pretrained language models, we also fine-tune the randomly initialized \textsc{HyTrel} model for comparison. In this section, we introduce the evaluation tasks and the datasets. We choose the following four tasks that rely solely on the table representations since we want to test the task-agnostic representation power of our model and avoid training separate encoders for texts (e.g., questions in table QA tasks) or decoders for generations.
As mentioned, our encoder can be used in all these scenarios and we leave its evaluation in other table-related tasks as future work.

\noindent {\bf Column Type Annotation } (CTA) task aims to annotate the semantic types of a column and is an important task in table understanding which can help many knowledge discovery tasks such as entity recognition and entity linking. We use the column representations from the final layer of  \textsc{HyTrel} with their corresponding hyperedge representations for making predictions. We evaluate \textsc{HyTrel} on the TURL-CTA dataset constructed by~\citet{deng2020turl}.


\noindent {\bf Column Property Annotation} (CPA) task aims to map column pairs from a table to relations in knowledge graphs. It is an important task aimed at extracting structured knowledge from tables. We use the dataset TURL-CPA constructed by ~\citet{deng2020turl} for evaluation. 

\noindent {\bf Table Type Detection} (TTD) task aims to annotate the semantic type of a table based on its content. We construct a dataset using a subset from the public~\href{http:\\webdatacommons.org/structureddata/schemaorgtables/}{WDC Schema.org Table Corpus}.


\noindent {\bf Table Similarity Prediction} (TSP) task aims at predicting the semantic similarity between tables and then classifying a table pair as similar or dissimilar. We use the PMC dataset proposed by~\citet{tabsim_inproceedings} for evaluation.

\begin{table*}[h] 
\small

\setlength{\tabcolsep}{5pt}
    \small
    \centering
    \begin{tabular}{l|c|c }
    \toprule
    \toprule
    \multirow{1}{*}{ Systems } & Column Type Annotation  & Column Property Annotation  \\
    \midrule

Sherlock & 88.40 / 70.55 / 78.47  &   -  \\ 
BERT$_{base}$ & - & 91.18 / 90.69 / 90.94 \\

TURL + metadata &   92.75 / 92.63 / 92.69 &  92.90 / 93.80 / 93.35  \\ 

Doduo + metadata & 93.25 / 92.34 / 92.79 & 91.20 / 94.50 / 92.82  \\
\midrule



TaBERT$_{base}$\textit{(K=1)} &  91.40$_{\pm\text{0.06}}$ / 89.49$_{\pm\text{0.21}}$ / 90.43$_{\pm\text{0.11}}$ &  92.31$_{\pm\text{0.24}}$ /	90.42$_{\pm\text{0.53}}$ /	91.36$_{\pm\text{0.30}}$    \\ 

{\hspace{1em} \textit{w/o} Pretrain }   & 90.00$_{\pm\text{0.14}}$ / 85.50$_{\pm\text{0.09}}$ / 87.70$_{\pm\text{0.10}}$ & 89.74$_{\pm\text{0.40}}$ /	68.74$_{\pm\text{0.93}}$ /	77.84$_{\pm\text{0.64}}$   \\

TaBERT$_{base}$\textit{(K=3)} & 91.63$_{\pm\text{0.21}}$ /	91.12$_{\pm\text{0.25}}$ /	91.37$_{\pm\text{0.08}}$ & 92.49$_{\pm\text{0.18}}$ /	92.49$_{\pm\text{0.22}}$ /	92.49$_{\pm\text{0.10}}$  \\

{\hspace{1em} \textit{w/o} Pretrain }   & 90.77$_{\pm\text{0.11}}$ / 87.23$_{\pm\text{0.22}}$ / 88.97$_{\pm\text{0.12}}$ & 90.10$_{\pm\text{0.17}}$ /	84.83$_{\pm\text{0.89}}$ /	87.38$_{\pm\text{0.48}}$   \\

    


\midrule

  \textsc{HyTrel} \textit{w/o} Pretrain           & {\bf 92.92}$_{\pm\text{0.11}}$ / 92.50$_{\pm\text{0.10}}$ / 92.71$_{\pm\text{0.08}}$	& 92.85$_{\pm\text{0.35}}$ / 91.50$_{\pm\text{0.54}}$ / 92.17$_{\pm\text{0.38}}$    \\

\textsc{HyTrel} \textit{w/} ELECTRA       &  92.85$_{\pm\text{0.21}}$ / {\bf 94.21}$_{\pm\text{0.09}}$ / {\bf 93.53}$_{\pm\text{0.10}}$ & 92.88$_{\pm\text{0.24}}$ / {\bf 94.07}$_{\pm\text{0.27}}$ / {\bf 93.48}$_{\pm\text{0.12}}$  \\ 

\textsc{HyTrel} \textit{w/} Contrastive       &  92.71$_{\pm\text{0.20}}$ / 93.24$_{\pm\text{0.08}}$ / 92.97$_{\pm\text{0.13}}$ & {\bf 93.01}$_{\pm\text{0.40}}$ / 93.16$_{\pm\text{0.40}}$  / 93.09$_{\pm\text{0.17}}$     \\ 
         \bottomrule
         \bottomrule

\end{tabular}


\caption{\small Test results on the CTA and CPA tasks (Precision/Recall/F1 Scores,\%). The results of TaBERT and \textsc{HyTrel} are from the average of 5 system runs with different random seeds. For fair comparisons, we use the results of TURL and Doduo with metadata, i.e., captions and headers.}
    \label{tab:column}
\end{table*}

\begin{table*}[t] 
\small

\setlength{\tabcolsep}{5pt}
    \small
    \centering
    \begin{tabular}{l|c|cc }
    \toprule
    \toprule
    \multirow{2}{*}{ Systems } & Table Type Detection  & \multicolumn{2}{c}{Table Similarity Prediction}  \\
    \cmidrule{2-4}
    \begin{tabular}[l] {@{}l@{}} \end{tabular} & Accuracy  &
    \begin{tabular}[c] {@{}c@{}} Precision/Recall/F1 \end{tabular} & 
    \begin{tabular}[c] {@{}c@{}} Accuracy \end{tabular} \\
    \midrule

TFIDF+Glove+MLP & - &87.36 / 83.81 / 84.47  & 85.06    \\
TabSim & - & 88.65 / 85.45 / 86.13 & 87.05 \\ 

\midrule



TaBERT$_{base}$\textit{(K=1)} & 93.11$_{\pm\text{0.31}}$ & 		87.04$_{\pm\text{0.64}}$ / 85.34$_{\pm\text{0.93}}$ /	86.18$_{\pm\text{1.13}}$ & 87.35$_{\pm\text{1.42}}$  \\ 

{\hspace{1em} \textit{w/o} Pretrain }   & 85.04$_{\pm\text{0.41}}$ &	33.61$_{\pm\text{12.70}}$ /	50.31$_{\pm\text{12.75}}$ /	40.30$_{\pm\text{12.03}}$  & 63.45$_{\pm\text{10.11}}$  \\

TaBERT$_{base}$\textit{(K=3)} & 95.15$_{\pm\text{0.14}}$  &	87.76$_{\pm\text{0.64}}$ / 86.97$_{\pm\text{0.59}}$ / 87.36$_{\pm\text{0.95}}$  & 88.29$_{\pm\text{0.98}}$  \\ 

{\hspace{1em} \textit{w/o} Pretrain }   & 89.88$_{\pm\text{0.26}}$ &	82.96$_{\pm\text{1.84}}$ /	81.16$_{\pm\text{1.45}}$ /	82.05$_{\pm\text{1.02}}$ &  82.57$_{\pm\text{1.20}}$ \\







\midrule

\textsc{HyTrel} \textit{w/o} Pretrain           & 93.84$_{\pm\text{0.17}}$ &	88.94$_{\pm\text{1.83}}$ /	85.72$_{\pm\text{1.52}}$ /	87.30$_{\pm\text{1.02}}$ & 88.38$_{\pm\text{1.43}}$  \\

\textsc{HyTrel} \textit{w/} ELECTRA       &  {\bf 95.81}$_{\pm\text{0.19}}$ &	87.35$_{\pm\text{0.42}}$ /	87.29$_{\pm\text{0.84}}$ /	87.32$_{\pm\text{0.50}}$ & 88.29$_{\pm\text{0.49}}$ \\ 

\textsc{HyTrel} \textit{w/} Contrastive       & 94.52$_{\pm\text{0.30}}$ &  	{\bf 89.41}$_{\pm\text{0.58}}$ /	{\bf 89.10}$_{\pm\text{0.90}}$ /	{\bf 89.26}$_{\pm\text{0.53}}$ & {\bf 90.12}$_{\pm\text{0.49}}$   \\  \bottomrule

    


         \bottomrule

\end{tabular}
\caption{\small Test results on the TTD (Accuracy Score,\%) and TSP (Precision/Recall/F1 Scores,\%) tasks. The results of TaBERT and \textsc{HyTrel} are from the average of 5 system runs with different random seeds. }
    \label{tab:table}
\end{table*}

\subsection{Baselines}
\noindent {\bf TaBERT}~\citep{yin20acl} is a representative TaLM that flattens the tables into sequences and jointly learns representations for sentences and tables by pretraining the model from the BERT checkpoints. \textit{K=1} and \textit{K=3} are the two variants based on the number of rows used. 

\noindent {\bf TURL}~\citep{deng2020turl} is another representative TaLM that also flattens the tables into sequences and pretrains from TinyBERT~\citep{jiao-etal-2020-tinybert} checkpoints. It introduces a vision matrix to incorporate table structure into the representations.  

\noindent {\bf  Doduo}~\citep{10.1145/3514221.3517906} is a state-of-the-art column annotation system that fine-tunes the BERT and uses table serialization to incorporate table content.



\subsection{Main Results}

The results are presented in Tables~\ref{tab:column} and~\ref{tab:table}. Overall, \textsc{HyTrel} consistently outperforms the baselines and achieve superior performance. 
A salient observation is that our model (even without pretraining) can achieve close to state-of-the-art performance.
In comparison, we notice that the performance slumps significantly for TaBERT without pretraining. This phenomenon empirically demonstrates the advantages of modeling the table structures as hypergraphs over the other methods that we compare. 


Additionally, we observe that the two pretraining objectives help different tasks in different ways. For the CTA, CPA, and TTD tasks, the two objectives can help \textsc{HyTrel} further improve its performance. In general, the ELECTRA objective performs better than the contrastive objective.
These results are also in line with the representation analysis in Section~\ref{sec:tsne} where we observe that the ELECTRA objective tends to learn table structure better than the contrastive objective. However, for the TSP task, we observe that the contrastive objective can help the \textsc{HyTrel} model while the ELECTRA objective fails to bring any improvement. One possible reason for the ineffectiveness of the ELECTRA objective could be its inability to transfer well across domains. \textsc{HyTrel} pretrained with tables from Wikipedia and Common Crawl could not transfer well to the medical domain PMC dataset. As for the improvement observed from the contrastive objective, the reason could be that contrastive learning that uses similarity metrics in the objective function can naturally help the similarity prediction task.

{\bf Scalability:} As stated in Section~\ref{sec:exp}, we have truncated large tables during pretraining. However, this truncation does not hinder the ability of \textsc{HyTrel} to handle large table inputs in downstream tasks. In Appendix~\ref{appdx:tab_size}, we present additional experiments demonstrating that: (a) \textsc{HyTrel} can effectively process tables of any size as inputs, and (b) down-sampling can be a favorable strategy when working with large input tables, significantly improving efficiency without compromising performance. 


\section{Qualitative Analysis}
\label{sec:bibtex}

\begin{wrapfigure}{r}{0.45\textwidth}
\begin{center}
\includegraphics[width=150pt]{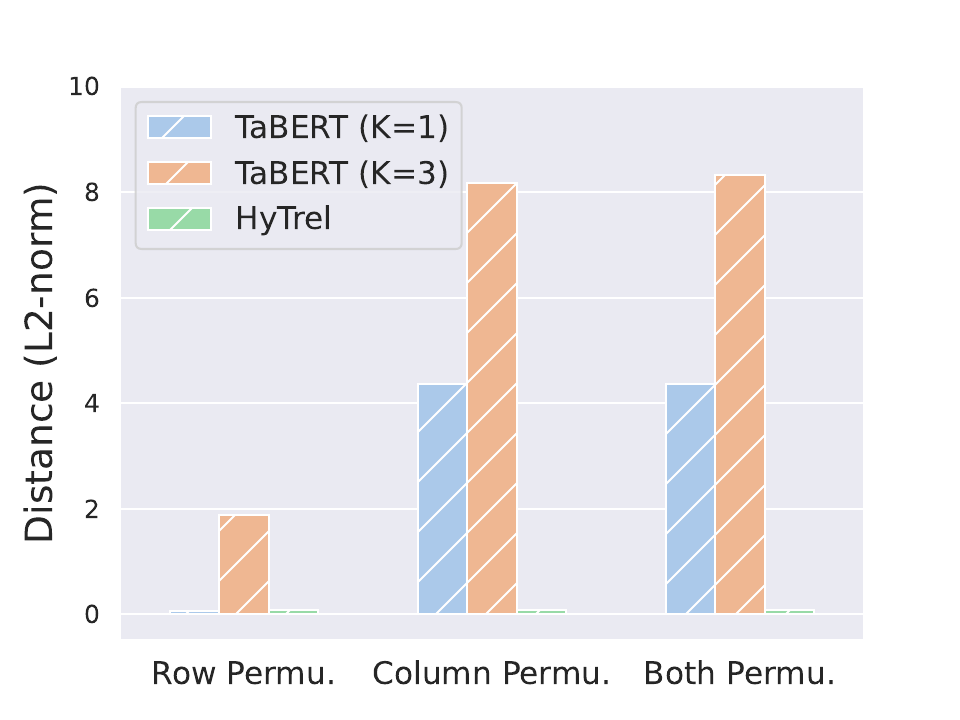}
\caption{\small Average distance between table element representations before and after permutations. The \textsc{HyTrel} is immune to the permutations while the TaBERT is sensitive to them.}
\label{fig:robust}
\end{center}
\end{wrapfigure}

\subsection{\textsc{HyTrel} Learns Permutation Robust Representations }

We sample 5,000 tables from the validation data for analysis. We  analyze the impact of applying different permutations to a table, including permuting rows, columns, and both rows/columns independently.

Toward our analysis, we measure the Euclidean distance (L2 Norm) of the representations (cells, rows, columns and tables). As shown in Figure~\ref{fig:robust}, the distance is almost always 0 for the \textsc{HyTrel} model because of its explicit invariance-preserving property. 
On the other hand, for the TaBERT model, the distance is not trivial. 
We observe that when more rows (K=3) are enabled, the value of the L2 norm increases as we introduce different permutations. 
Moreover, we also observe that permuting the columns has a greater impact on the L2 norm than the row permutations. 
A combination of rows and columns permutations has the largest impact among all three actions. 
Note that when K=1 with just one row, the effect of row permutation is disabled for TaBERT.

\subsection{\textsc{HyTrel} Learns the Underlying Hierarchical Table Structure}
\label{sec:tsne}
Next, we demonstrate that the \textsc{HyTrel} model has learned the hierarchical table structure into its representations, as we target at. We use t-SNE \citep{van2008visualizing} for the visualization of different table elements from the same 5,000 validation tables, as shown in  Figure~\ref{fig:tsne}. 




\begin{figure*}[h]
\begin{center}
\includegraphics[width=360pt]{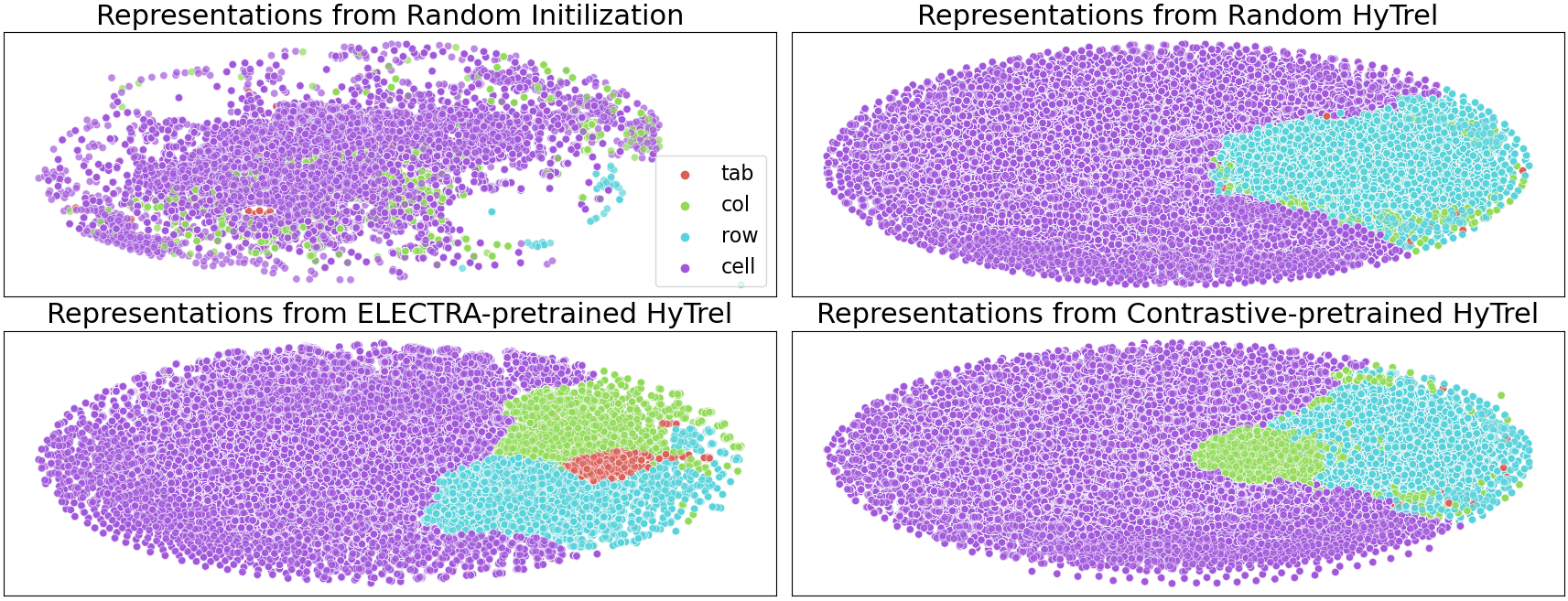}
\caption{t-SNE visualization of the representations learned by \textsc{HyTrel}. `tab', `col', `row', and `cell' are the representations for different table elements: tables, columns, rows, and cells. }
\label{fig:tsne}
\end{center}
\end{figure*}

We observe that with random initializations, different table elements cannot be distinguished properly. 
After the encoding of the randomly initialized \textsc{HyTrel} model, we start to observe noticeable differences for different table elements in the visualization space.
Notably, the individual cell representations start to concentrate together and can be distinguished from high-level table elements (tables, columns, and rows) which occupy their separate places in the space of representations. 
We also notice that, by pretraining the \textsc{HyTrel} with the ELECTRA objective, all table elements can be well separated, showing that it incorporates the hierarchical table structure into its latent representations. 
As for the contrastive pretraining, we  see that it can distinguish columns from rows as compared with randomly initialized \textsc{HyTrel}, but could not to well separate the table representations in comparison with the ELECTRA pretraining. 
This also partially explains the better performance of the ELECTRA pretraining in the CTA, CPA and TTD tasks in contrast to the contrastive pretraining.             
%
%
%

\subsection{\textsc{HyTrel} Demonstrates Effective Pretraining by Capturing the Table Structures}
\label{appdx:eff}

 \begin{figure*}[ht]
\begin{center}
\includegraphics[width=400pt]{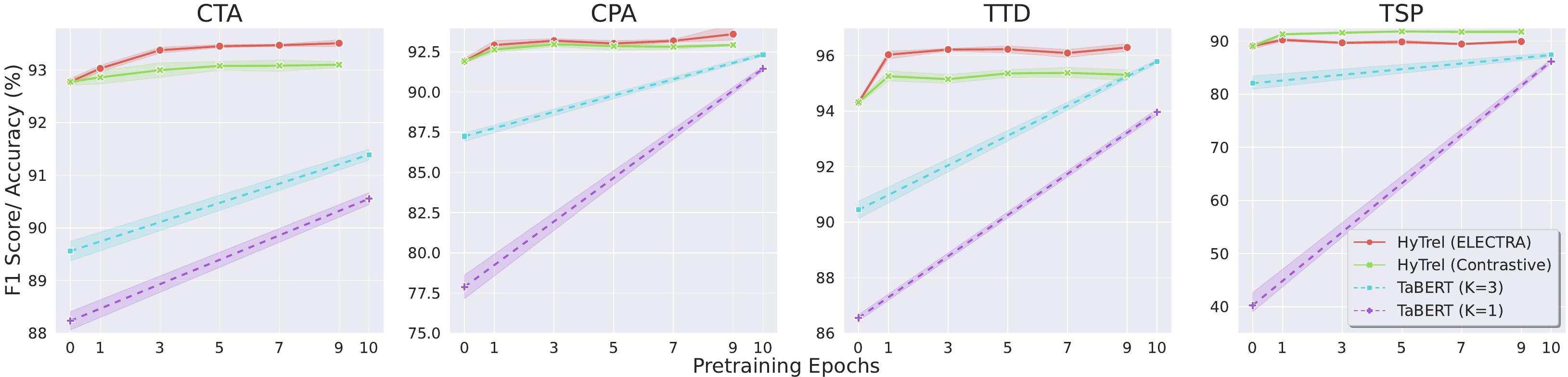}

\caption{\small Performance with different pretraining checkpoints on the validation set of the four tasks. 
For the TaBERT models, we can only access the randomly initialized and fully pretrained (10 epochs) checkpoints. All results are from 5 system runs with different random seeds.
}
\label{fig:eff}
\end{center}
\end{figure*}

Our evaluation shows that the \textsc{HyTrel} model can perform well without pretraining, which demonstrates its training efficiency by modeling table structures. Here we further analyze how much pretraining is required for \textsc{HyTrel} to further enhance its performance, as compared with the baseline model. 
We plot the validation performance of the tasks evaluated at different pretraining checkpoints in Figure~\ref{fig:eff}.   

Overall, we can observe that the performance of \textsc{HyTrel} improves drastically during the first several pretraining epochs, and then saturates at about 5 epochs. With the minimal pretraining, \textsc{HyTrel} can outperform the fully pretrained TaBERT models. This demonstrates that our model does not require extensive pretraining to further improve its performance in contrast with previous TaLMs (e.g., TURL for 100 epochs, TaBERT for 10 epochs).
%
%
%
Besides, we also observe from the curves that the ELECTRA objective consistently outperforms the contrastive objective for the CTA, CPA, and TTD tasks, but under-performs on the TSP task. 
Also, the ELECTRA objective has a negative impact on the TSP task when pretrained for longer duration, which is in line with our previous findings.


\section{Related Work}


There are two avenues of research that have studied in tabular data representation learning. The first group of studies focus on predicting labels (essentially one row) for classification and regression problems, using row information and column schema as input\citep{huang2020tabtransformer,arik2021tabnet, somepalli2021saint,gorishniy2021revisiting,grinsztajn2022why,wang2022transtab,Du2022,wydmanski2023hypertab,10.5555/3618408.3619291}. These studies use gradient descent-based end-to-end learning and aim to outperform tree-based models through task-specific model pretraining and fine-tuning.

The second group of studies proposes TaLMs to retrieve task-agnostic tabular data representations for different downstream table understanding tasks. Drawing inspiration from textual Language Models like BERT~\citep{devlin-etal-2019-bert}, many works~\citep{herzig-etal-2020-tapas,yin20acl,deng2020turl,iida-etal-2021-tabbie,eisenschlos-etal-2021-mate} serialize tables to a sequence of tokens, leveraging existing checkpoints and textual self-supervised objectives. However, the representations of the tables can not only be learned from table content and by utilizing table structures, similar to other forms of semi-structured data like code and HTML. Some contemporary works have noticed the importance of the table structures and introduce many techniques to learn a certain aspect of them, such as masking~\citep{deng2020turl}, coordinates~\citep{10.1145/3447548.3467434,dash2022permutation}, and attention bias~\citep{yang-etal-2022-tableformer}. Our work belongs to this second group of studies and we propose to use hypergraphs to comprehensively model the rich table structures, and this is close to previous graph-based neural networks \citep{ mueller-etal-2019-answering,10.1145/3404835.3462909,wang2021tcn} where tables have been structured as graphs to incorporate row and column order information.


Table representation learning that focuses on joint text and table understanding is a separate field of research that partially overlaps with our work. Among them, some work~\citep{herzig-etal-2020-tapas,shi2022generation,herzig2021open,glass2021capturing,yang-etal-2022-tableformer} specialize in question-answering (QA) on tables and they jointly train a model that takes the question and the table structure as input together, allowing the pretraining to attend to the interactions of both texts and tables and boosting the table-based QA tasks. Another branch of joint text and table understanding work focuses on text generation from tables\citep{parikh2020totto,yoran2021turning,wang-etal-2022-robust,andrejczuk2022table}, relying on an encoder-decoder model like T5~\citep{2020t5} that can encode tables and decode free-form texts. In contrast to these studies, our work centers on the importance of structures in tables for table representation only, without extra text encoding or decoding.   



Learning on hypergraphs has gone through a series of evolution~\citep{agarwal2006higher,yadati2019hypergcn,arya2020hypersage} in the way the hypergraph structure is modeled using neural networks layers. However, many of them collapse the hypergraph into a fully connected graph by clique expansion and cannot preserve the high-order interaction among the nodes. The recent development of permutation invariant networks~\citep{NIPS2017_f22e4747,lee2019set} has enabled high-order interactions on the hypergraphs~\citep{chien2022you} that uses parameterized multi-set functions to model dual relations from node to hyperedges and vice versa. Closely related to the latest advancement, our \textsc{HyTrel} model adopts a similar neural message passing on hypergraphs to preserve the invariance property and high-order interactions of tables.

\section{Limitations}



The proposed \textsc{HyTrel} is a table encoder, and by itself cannot handle joint text and table understanding tasks like table QA and table-to-text generation.
While it's possible to add text encoders or decoders for these tasks, it can potentially introduce additional factors that may complicate evaluating our hypothesis about the usefulness of modeling structural table properties. 
Moreover, the current model structure is designed for tables with simple column structures, like prior TaLMs, and cannot handle tables with complex hierarchical column structures. Additionally, \textsc{HyTrel} does not consider cross-table relations. Although we believe the hypergraph can generalize to model complicated column structures and cross-table interactions, we leave these aspects for future research.


\section{Conclusion}

In this work, we propose a tabular language model \textsc{HyTrel} that models tables as hypergraphs. It can incorporate the permutation invariances and table structures into the table representations. The evaluation on four table-related tasks demonstrates the advantages of learning these table properties and show that it can consistently achieve superior performance over the competing baselines. Our theoretical and qualitative analyses also support the effectiveness of learning the structural properties.

\bibliographystyle{unsrtnat}
\bibliography{ref}

\begin{thebibliography}{53}
\providecommand{\natexlab}[1]{#1}
\providecommand{\url}[1]{\texttt{#1}}
\expandafter\ifx\csname urlstyle\endcsname\relax
  \providecommand{\doi}[1]{doi: #1}\else
  \providecommand{\doi}{doi: \begingroup \urlstyle{rm}\Url}\fi

\bibitem[Yin et~al.(2020)Yin, Neubig, tau Yih, and Riedel]{yin20acl}
Pengcheng Yin, Graham Neubig, Wen tau Yih, and Sebastian Riedel.
\newblock Ta{BERT}: Pretraining for joint understanding of textual and tabular
  data.
\newblock In \emph{Annual Conference of the Association for Computational
  Linguistics (ACL)}, July 2020.

\bibitem[Yang et~al.(2022)Yang, Gupta, Upadhyay, He, Goel, and
  Paul]{yang-etal-2022-tableformer}
Jingfeng Yang, Aditya Gupta, Shyam Upadhyay, Luheng He, Rahul Goel, and Shachi
  Paul.
\newblock {T}able{F}ormer: Robust transformer modeling for table-text encoding.
\newblock In \emph{Proceedings of the 60th Annual Meeting of the Association
  for Computational Linguistics (Volume 1: Long Papers)}, pages 528--537,
  Dublin, Ireland, May 2022. Association for Computational Linguistics.
\newblock \doi{10.18653/v1/2022.acl-long.40}.
\newblock URL \url{https://aclanthology.org/2022.acl-long.40}.

\bibitem[Herzig et~al.(2020)Herzig, Nowak, M{\"u}ller, Piccinno, and
  Eisenschlos]{herzig-etal-2020-tapas}
Jonathan Herzig, Pawel~Krzysztof Nowak, Thomas M{\"u}ller, Francesco Piccinno,
  and Julian Eisenschlos.
\newblock {T}a{P}as: Weakly supervised table parsing via pre-training.
\newblock In \emph{Proceedings of the 58th Annual Meeting of the Association
  for Computational Linguistics}, pages 4320--4333, Online, July 2020.
  Association for Computational Linguistics.
\newblock \doi{10.18653/v1/2020.acl-main.398}.
\newblock URL \url{https://aclanthology.org/2020.acl-main.398}.

\bibitem[Deng et~al.(2020)Deng, Sun, Lees, Wu, and Yu]{deng2020turl}
Xiang Deng, Huan Sun, Alyssa Lees, You Wu, and Cong Yu.
\newblock Turl: table understanding through representation learning.
\newblock \emph{Proceedings of the VLDB Endowment}, 14\penalty0 (3):\penalty0
  307--319, 2020.

\bibitem[Iida et~al.(2021)Iida, Thai, Manjunatha, and
  Iyyer]{iida-etal-2021-tabbie}
Hiroshi Iida, Dung Thai, Varun Manjunatha, and Mohit Iyyer.
\newblock {TABBIE}: Pretrained representations of tabular data.
\newblock In \emph{Proceedings of the 2021 Conference of the North American
  Chapter of the Association for Computational Linguistics: Human Language
  Technologies}, pages 3446--3456, Online, June 2021. Association for
  Computational Linguistics.
\newblock \doi{10.18653/v1/2021.naacl-main.270}.
\newblock URL \url{https://aclanthology.org/2021.naacl-main.270}.

\bibitem[Wang et~al.(2021{\natexlab{a}})Wang, Dong, Jia, Li, Fu, Han, and
  Zhang]{10.1145/3447548.3467434}
Zhiruo Wang, Haoyu Dong, Ran Jia, Jia Li, Zhiyi Fu, Shi Han, and Dongmei Zhang.
\newblock Tuta: Tree-based transformers for generally structured table
  pre-training.
\newblock In \emph{Proceedings of the 27th ACM SIGKDD Conference on Knowledge
  Discovery \&amp; Data Mining}, KDD '21, page 1780–1790, New York, NY, USA,
  2021{\natexlab{a}}. Association for Computing Machinery.
\newblock ISBN 9781450383325.
\newblock \doi{10.1145/3447548.3467434}.
\newblock URL \url{https://doi.org/10.1145/3447548.3467434}.

\bibitem[Clark et~al.(2020)Clark, Luong, Le, and Manning]{clark2020electra}
Kevin Clark, Minh-Thang Luong, Quoc~V. Le, and Christopher~D. Manning.
\newblock {ELECTRA}: Pre-training text encoders as discriminators rather than
  generators.
\newblock In \emph{ICLR}, 2020.
\newblock URL \url{https://openreview.net/pdf?id=r1xMH1BtvB}.

\bibitem[Wei et~al.(2022)Wei, You, Chen, Shen, He, and
  Wang]{wei2022augmentations}
Tianxin Wei, Yuning You, Tianlong Chen, Yang Shen, Jingrui He, and Zhangyang
  Wang.
\newblock Augmentations in hypergraph contrastive learning: Fabricated and
  generative.
\newblock \emph{arXiv preprint arXiv:2210.03801}, 2022.

\bibitem[Chien et~al.(2022)Chien, Pan, Peng, and Milenkovic]{chien2022you}
Eli Chien, Chao Pan, Jianhao Peng, and Olgica Milenkovic.
\newblock You are allset: A multiset function framework for hypergraph neural
  networks.
\newblock In \emph{International Conference on Learning Representations}, 2022.
\newblock URL \url{https://openreview.net/forum?id=hpBTIv2uy_E}.

\bibitem[Vaswani et~al.(2017)Vaswani, Shazeer, Parmar, Uszkoreit, Jones, Gomez,
  Kaiser, and Polosukhin]{NIPS2017_3f5ee243}
Ashish Vaswani, Noam Shazeer, Niki Parmar, Jakob Uszkoreit, Llion Jones,
  Aidan~N Gomez, \L~ukasz Kaiser, and Illia Polosukhin.
\newblock Attention is all you need.
\newblock In I.~Guyon, U.~Von Luxburg, S.~Bengio, H.~Wallach, R.~Fergus,
  S.~Vishwanathan, and R.~Garnett, editors, \emph{Advances in Neural
  Information Processing Systems}, volume~30. Curran Associates, Inc., 2017.
\newblock URL
  \url{https://proceedings.neurips.cc/paper/2017/file/3f5ee243547dee91fbd053c1c4a845aa-Paper.pdf}.

\bibitem[Ba et~al.(2016)Ba, Kiros, and Hinton]{Ba2016LayerN}
Jimmy Ba, Jamie~Ryan Kiros, and Geoffrey~E. Hinton.
\newblock Layer normalization.
\newblock \emph{ArXiv}, abs/1607.06450, 2016.

\bibitem[He et~al.(2016)He, Zhang, Ren, and Sun]{7780459}
Kaiming He, Xiangyu Zhang, Shaoqing Ren, and Jian Sun.
\newblock Deep residual learning for image recognition.
\newblock In \emph{2016 IEEE Conference on Computer Vision and Pattern
  Recognition (CVPR)}, pages 770--778, 2016.
\newblock \doi{10.1109/CVPR.2016.90}.

\bibitem[Zaheer et~al.(2017)Zaheer, Kottur, Ravanbakhsh, Poczos, Salakhutdinov,
  and Smola]{NIPS2017_f22e4747}
Manzil Zaheer, Satwik Kottur, Siamak Ravanbakhsh, Barnabas Poczos, Russ~R
  Salakhutdinov, and Alexander~J Smola.
\newblock Deep sets.
\newblock In I.~Guyon, U.~Von Luxburg, S.~Bengio, H.~Wallach, R.~Fergus,
  S.~Vishwanathan, and R.~Garnett, editors, \emph{Advances in Neural
  Information Processing Systems}, volume~30. Curran Associates, Inc., 2017.
\newblock URL
  \url{https://proceedings.neurips.cc/paper/2017/file/f22e4747da1aa27e363d86d40ff442fe-Paper.pdf}.

\bibitem[Lee et~al.(2019)Lee, Lee, Kim, Kosiorek, Choi, and Teh]{lee2019set}
Juho Lee, Yoonho Lee, Jungtaek Kim, Adam Kosiorek, Seungjin Choi, and Yee~Whye
  Teh.
\newblock Set transformer: A framework for attention-based
  permutation-invariant neural networks.
\newblock In \emph{Proceedings of the 36th International Conference on Machine
  Learning}, pages 3744--3753, 2019.

\bibitem[Devlin et~al.(2019)Devlin, Chang, Lee, and
  Toutanova]{devlin-etal-2019-bert}
Jacob Devlin, Ming-Wei Chang, Kenton Lee, and Kristina Toutanova.
\newblock {BERT}: Pre-training of deep bidirectional transformers for language
  understanding.
\newblock In \emph{Proceedings of the 2019 Conference of the North {A}merican
  Chapter of the Association for Computational Linguistics: Human Language
  Technologies, Volume 1 (Long and Short Papers)}, pages 4171--4186,
  Minneapolis, Minnesota, June 2019. Association for Computational Linguistics.
\newblock \doi{10.18653/v1/N19-1423}.
\newblock URL \url{https://aclanthology.org/N19-1423}.

\bibitem[Tyshkevich and Zverovich(1996)]{tyshkevich1996line}
RI~Tyshkevich and Vadim~E Zverovich.
\newblock Line hypergraphs.
\newblock \emph{Discrete Mathematics}, 161\penalty0 (1-3):\penalty0 265--283,
  1996.

\bibitem[Lyle et~al.(2020)Lyle, van~der Wilk, Kwiatkowska, Gal, and
  Bloem-Reddy]{lyle2020benefits}
Clare Lyle, Mark van~der Wilk, Marta Kwiatkowska, Yarin Gal, and Benjamin
  Bloem-Reddy.
\newblock On the benefits of invariance in neural networks.
\newblock \emph{arXiv preprint arXiv:2005.00178}, 2020.

\bibitem[van~den Oord et~al.(2018)van~den Oord, Li, and
  Vinyals]{DBLP:journals/corr/abs-1807-03748}
A{\"{a}}ron van~den Oord, Yazhe Li, and Oriol Vinyals.
\newblock Representation learning with contrastive predictive coding.
\newblock \emph{CoRR}, abs/1807.03748, 2018.
\newblock URL \url{http://arxiv.org/abs/1807.03748}.

\bibitem[Habibi et~al.(2020)Habibi, Starlinger, and
  Leser]{tabsim_inproceedings}
Maryam Habibi, Johannes Starlinger, and Ulf Leser.
\newblock Tabsim: A siamese neural network for accurate estimation of table
  similarity.
\newblock pages 930--937, 12 2020.
\newblock \doi{10.1109/BigData50022.2020.9378077}.

\bibitem[Jiao et~al.(2020)Jiao, Yin, Shang, Jiang, Chen, Li, Wang, and
  Liu]{jiao-etal-2020-tinybert}
Xiaoqi Jiao, Yichun Yin, Lifeng Shang, Xin Jiang, Xiao Chen, Linlin Li, Fang
  Wang, and Qun Liu.
\newblock {T}iny{BERT}: Distilling {BERT} for natural language understanding.
\newblock In \emph{Findings of the Association for Computational Linguistics:
  EMNLP 2020}, pages 4163--4174, Online, November 2020. Association for
  Computational Linguistics.
\newblock \doi{10.18653/v1/2020.findings-emnlp.372}.
\newblock URL \url{https://aclanthology.org/2020.findings-emnlp.372}.

\bibitem[Suhara et~al.(2022)Suhara, Li, Li, Zhang, Demiralp, Chen, and
  Tan]{10.1145/3514221.3517906}
Yoshihiko Suhara, Jinfeng Li, Yuliang Li, Dan Zhang, \c{C}a\u{g}atay Demiralp,
  Chen Chen, and Wang-Chiew Tan.
\newblock \emph{Annotating Columns with Pre-trained Language Models}.
\newblock Association for Computing Machinery, 2022.
\newblock ISBN 9781450392495.
\newblock URL \url{https://doi.org/10.1145/3514221.3517906}.

\bibitem[Van~der Maaten and Hinton(2008)]{van2008visualizing}
Laurens Van~der Maaten and Geoffrey Hinton.
\newblock Visualizing data using t-sne.
\newblock \emph{Journal of machine learning research}, 9\penalty0 (11), 2008.

\bibitem[Huang et~al.(2020)Huang, Khetan, Cvitkovic, and
  Karnin]{huang2020tabtransformer}
Xin Huang, Ashish Khetan, Milan Cvitkovic, and Zohar Karnin.
\newblock Tabtransformer: Tabular data modeling using contextual embeddings.
\newblock \emph{arXiv preprint arXiv:2012.06678}, 2020.

\bibitem[Arik and Pfister(2021)]{arik2021tabnet}
Sercan~{\"O} Arik and Tomas Pfister.
\newblock Tabnet: Attentive interpretable tabular learning.
\newblock In \emph{Proceedings of the AAAI Conference on Artificial
  Intelligence}, volume~35, pages 6679--6687, 2021.

\bibitem[Somepalli et~al.(2021)Somepalli, Goldblum, Schwarzschild, Bruss, and
  Goldstein]{somepalli2021saint}
Gowthami Somepalli, Micah Goldblum, Avi Schwarzschild, C~Bayan Bruss, and Tom
  Goldstein.
\newblock Saint: Improved neural networks for tabular data via row attention
  and contrastive pre-training.
\newblock \emph{arXiv preprint arXiv:2106.01342}, 2021.

\bibitem[Gorishniy et~al.(2021)Gorishniy, Rubachev, Khrulkov, and
  Babenko]{gorishniy2021revisiting}
Yury Gorishniy, Ivan Rubachev, Valentin Khrulkov, and Artem Babenko.
\newblock Revisiting deep learning models for tabular data.
\newblock In \emph{{NeurIPS}}, 2021.

\bibitem[Grinsztajn et~al.(2022)Grinsztajn, Oyallon, and
  Varoquaux]{grinsztajn2022why}
Leo Grinsztajn, Edouard Oyallon, and Gael Varoquaux.
\newblock Why do tree-based models still outperform deep learning on typical
  tabular data?
\newblock In \emph{Thirty-sixth Conference on Neural Information Processing
  Systems Datasets and Benchmarks Track}, 2022.
\newblock URL \url{https://openreview.net/forum?id=Fp7__phQszn}.

\bibitem[Wang and Sun(2022)]{wang2022transtab}
Zifeng Wang and Jimeng Sun.
\newblock Transtab: Learning transferable tabular transformers across tables.
\newblock In \emph{Advances in Neural Information Processing Systems}, 2022.

\bibitem[Du et~al.(2022)Du, Zhang, Zhou, Wang, Zhao, Jin, Gan, Zhang, and
  Wipf]{Du2022}
Kounianhua Du, Weinan Zhang, Ruiwen Zhou, Yangkun Wang, Xilong Zhao, Jiarui
  Jin, Quan Gan, Zheng Zhang, and David~Paul Wipf.
\newblock Learning enhanced representations for tabular data via neighborhood
  propagation.
\newblock In \emph{NeurIPS 2022}, 2022.
\newblock URL
  \url{https://www.amazon.science/publications/learning-enhanced-representations-for-tabular-data-via-neighborhood-propagation}.

\bibitem[Wydma{\'n}ski et~al.(2023)Wydma{\'n}ski, Bulenok, and
  {\'S}mieja]{wydmanski2023hypertab}
Witold Wydma{\'n}ski, Oleksii Bulenok, and Marek {\'S}mieja.
\newblock Hypertab: Hypernetwork approach for deep learning on small tabular
  datasets.
\newblock \emph{arXiv preprint arXiv:2304.03543}, 2023.

\bibitem[Liu et~al.(2023)Liu, Berrevoets, Qian, and Van
  Der~Schaar]{10.5555/3618408.3619291}
Tennison Liu, Jeroen Berrevoets, Zhaozhi Qian, and Mihaela Van Der~Schaar.
\newblock Learning representations without compositional assumptions.
\newblock In \emph{Proceedings of the 40th International Conference on Machine
  Learning}, ICML'23. JMLR.org, 2023.

\bibitem[Eisenschlos et~al.(2021)Eisenschlos, Gor, M{\"u}ller, and
  Cohen]{eisenschlos-etal-2021-mate}
Julian Eisenschlos, Maharshi Gor, Thomas M{\"u}ller, and William Cohen.
\newblock {MATE}: Multi-view attention for table transformer efficiency.
\newblock In \emph{Proceedings of the 2021 Conference on Empirical Methods in
  Natural Language Processing}, pages 7606--7619, Online and Punta Cana,
  Dominican Republic, November 2021. Association for Computational Linguistics.
\newblock \doi{10.18653/v1/2021.emnlp-main.600}.
\newblock URL \url{https://aclanthology.org/2021.emnlp-main.600}.

\bibitem[Dash et~al.(2022)Dash, Bagchi, Mihindukulasooriya, and
  Gliozzo]{dash2022permutation}
Sarthak Dash, Sugato Bagchi, Nandana Mihindukulasooriya, and Alfio Gliozzo.
\newblock Permutation invariant strategy using transformer encoders for table
  understanding.
\newblock In \emph{Findings of the Association for Computational Linguistics:
  NAACL 2022}, pages 788--800, 2022.

\bibitem[Mueller et~al.(2019)Mueller, Piccinno, Shaw, Nicosia, and
  Altun]{mueller-etal-2019-answering}
Thomas Mueller, Francesco Piccinno, Peter Shaw, Massimo Nicosia, and Yasemin
  Altun.
\newblock Answering conversational questions on structured data without logical
  forms.
\newblock In \emph{Proceedings of the 2019 Conference on Empirical Methods in
  Natural Language Processing and the 9th International Joint Conference on
  Natural Language Processing (EMNLP-IJCNLP)}, pages 5902--5910, Hong Kong,
  China, November 2019. Association for Computational Linguistics.
\newblock \doi{10.18653/v1/D19-1603}.
\newblock URL \url{https://aclanthology.org/D19-1603}.

\bibitem[Wang et~al.(2021{\natexlab{b}})Wang, Sun, Chen, Pujara, and
  Szekely]{10.1145/3404835.3462909}
Fei Wang, Kexuan Sun, Muhao Chen, Jay Pujara, and Pedro Szekely.
\newblock Retrieving complex tables with multi-granular graph representation
  learning.
\newblock In \emph{Proceedings of the 44th International ACM SIGIR Conference
  on Research and Development in Information Retrieval}, SIGIR '21, page
  1472–1482, New York, NY, USA, 2021{\natexlab{b}}. Association for Computing
  Machinery.
\newblock ISBN 9781450380379.
\newblock \doi{10.1145/3404835.3462909}.
\newblock URL \url{https://doi.org/10.1145/3404835.3462909}.

\bibitem[Wang et~al.(2021{\natexlab{c}})Wang, Shiralkar, Lockard, Huang, Dong,
  and Jiang]{wang2021tcn}
Daheng Wang, Prashant Shiralkar, Colin Lockard, Binxuan Huang, Xin~Luna Dong,
  and Meng Jiang.
\newblock Tcn: Table convolutional network for web table interpretation.
\newblock In \emph{Proceedings of the Web Conference 2021}, pages 4020--4032,
  2021{\natexlab{c}}.

\bibitem[Shi et~al.(2022)Shi, Ng, Nan, Zhu, Wang, Jiang, Li, Chakravarti,
  Weidner, Xiang, et~al.]{shi2022generation}
Peng Shi, Patrick Ng, Feng Nan, Henghui Zhu, Jun Wang, Jiarong Jiang,
  Alexander~Hanbo Li, Rishav Chakravarti, Donald Weidner, Bing Xiang, et~al.
\newblock Generation-focused table-based intermediate pre-training for
  free-form question answering.
\newblock 2022.

\bibitem[Herzig et~al.(2021)Herzig, M{\"u}ller, Krichene, and
  Eisenschlos]{herzig2021open}
Jonathan Herzig, Thomas M{\"u}ller, Syrine Krichene, and Julian~Martin
  Eisenschlos.
\newblock Open domain question answering over tables via dense retrieval.
\newblock \emph{arXiv preprint arXiv:2103.12011}, 2021.

\bibitem[Glass et~al.(2021)Glass, Canim, Gliozzo, Chemmengath, Kumar,
  Chakravarti, Sil, Pan, Bharadwaj, and Fauceglia]{glass2021capturing}
Michael Glass, Mustafa Canim, Alfio Gliozzo, Saneem Chemmengath, Vishwajeet
  Kumar, Rishav Chakravarti, Avi Sil, Feifei Pan, Samarth Bharadwaj, and
  Nicolas~Rodolfo Fauceglia.
\newblock Capturing row and column semantics in transformer based question
  answering over tables.
\newblock \emph{arXiv preprint arXiv:2104.08303}, 2021.

\bibitem[Parikh et~al.(2020)Parikh, Wang, Gehrmann, Faruqui, Dhingra, Yang, and
  Das]{parikh2020totto}
Ankur~P Parikh, Xuezhi Wang, Sebastian Gehrmann, Manaal Faruqui, Bhuwan
  Dhingra, Diyi Yang, and Dipanjan Das.
\newblock Totto: A controlled table-to-text generation dataset.
\newblock \emph{arXiv preprint arXiv:2004.14373}, 2020.

\bibitem[Yoran et~al.(2021)Yoran, Talmor, and Berant]{yoran2021turning}
Ori Yoran, Alon Talmor, and Jonathan Berant.
\newblock Turning tables: Generating examples from semi-structured tables for
  endowing language models with reasoning skills.
\newblock \emph{arXiv preprint arXiv:2107.07261}, 2021.

\bibitem[Wang et~al.(2022)Wang, Xu, Szekely, and Chen]{wang-etal-2022-robust}
Fei Wang, Zhewei Xu, Pedro Szekely, and Muhao Chen.
\newblock Robust (controlled) table-to-text generation with structure-aware
  equivariance learning.
\newblock In \emph{Proceedings of the 2022 Conference of the North American
  Chapter of the Association for Computational Linguistics: Human Language
  Technologies}, pages 5037--5048, Seattle, United States, July 2022.
  Association for Computational Linguistics.
\newblock \doi{10.18653/v1/2022.naacl-main.371}.
\newblock URL \url{https://aclanthology.org/2022.naacl-main.371}.

\bibitem[Andrejczuk et~al.(2022)Andrejczuk, Eisenschlos, Piccinno, Krichene,
  and Altun]{andrejczuk2022table}
Ewa Andrejczuk, Julian~Martin Eisenschlos, Francesco Piccinno, Syrine Krichene,
  and Yasemin Altun.
\newblock Table-to-text generation and pre-training with tabt5.
\newblock \emph{arXiv preprint arXiv:2210.09162}, 2022.

\bibitem[Raffel et~al.(2020)Raffel, Shazeer, Roberts, Lee, Narang, Matena,
  Zhou, Li, and Liu]{2020t5}
Colin Raffel, Noam Shazeer, Adam Roberts, Katherine Lee, Sharan Narang, Michael
  Matena, Yanqi Zhou, Wei Li, and Peter~J. Liu.
\newblock Exploring the limits of transfer learning with a unified text-to-text
  transformer.
\newblock \emph{Journal of Machine Learning Research}, 21\penalty0
  (140):\penalty0 1--67, 2020.
\newblock URL \url{http://jmlr.org/papers/v21/20-074.html}.

\bibitem[Agarwal et~al.(2006)Agarwal, Branson, and Belongie]{agarwal2006higher}
Sameer Agarwal, Kristin Branson, and Serge Belongie.
\newblock Higher order learning with graphs.
\newblock In \emph{Proceedings of the 23rd international conference on Machine
  learning}, pages 17--24, 2006.

\bibitem[Yadati et~al.(2019)Yadati, Nimishakavi, Yadav, Nitin, Louis, and
  Talukdar]{yadati2019hypergcn}
Naganand Yadati, Madhav Nimishakavi, Prateek Yadav, Vikram Nitin, Anand Louis,
  and Partha Talukdar.
\newblock Hypergcn: A new method for training graph convolutional networks on
  hypergraphs.
\newblock \emph{Advances in neural information processing systems}, 32, 2019.

\bibitem[Arya et~al.(2020)Arya, Gupta, Rudinac, and Worring]{arya2020hypersage}
Devanshu Arya, Deepak~K Gupta, Stevan Rudinac, and Marcel Worring.
\newblock Hypersage: Generalizing inductive representation learning on
  hypergraphs.
\newblock \emph{arXiv preprint arXiv:2010.04558}, 2020.

\bibitem[Babai et~al.(1980)Babai, Erdo˝s, and Selkow]{doi:10.1137/0209047}
L\'{a}szl\'{o} Babai, Paul Erdo˝s, and Stanley~M. Selkow.
\newblock Random graph isomorphism.
\newblock \emph{SIAM Journal on Computing}, 9\penalty0 (3):\penalty0 628--635,
  1980.
\newblock \doi{10.1137/0209047}.
\newblock URL \url{https://doi.org/10.1137/0209047}.

\bibitem[Xu et~al.(2018)Xu, Hu, Leskovec, and Jegelka]{xu2018powerful}
Keyulu Xu, Weihua Hu, Jure Leskovec, and Stefanie Jegelka.
\newblock How powerful are graph neural networks?
\newblock \emph{arXiv preprint arXiv:1810.00826}, 2018.

\bibitem[Srinivasan et~al.(2021)Srinivasan, Zheng, and
  Karypis]{srinivasan2021learning}
Balasubramaniam Srinivasan, Da~Zheng, and George Karypis.
\newblock Learning over families of sets-hypergraph representation learning for
  higher order tasks.
\newblock In \emph{Proceedings of the 2021 SIAM International Conference on
  Data Mining (SDM)}, pages 756--764. SIAM, 2021.

\bibitem[Trabelsi et~al.(2022)Trabelsi, Chen, Zhang, Davison, and
  Heflin]{trabelsi22www}
Mohamed Trabelsi, Zhiyu Chen, Shuo Zhang, Brian~D Davison, and Jeff Heflin.
\newblock Stru{BERT}: Structure-aware bert for table search and matching.
\newblock In \emph{Proceedings of the ACM Web Conference}, WWW '22, 2022.

\bibitem[Hulsebos et~al.(2019)Hulsebos, Hu, Bakker, Zgraggen, Satyanarayan,
  Kraska, Demiralp, and Hidalgo]{Hulsebos:2019:SDL:3292500.3330993}
Madelon Hulsebos, Kevin Hu, Michiel Bakker, Emanuel Zgraggen, Arvind
  Satyanarayan, Tim Kraska, {\c{C}}a{\u{g}}atay Demiralp, and C{\'e}sar
  Hidalgo.
\newblock Sherlock: A deep learning approach to semantic data type detection.
\newblock In \emph{Proceedings of the 25th ACM SIGKDD International Conference
  on Knowledge Discovery \&\#38; Data Mining}. ACM, 2019.

\bibitem[Pennington et~al.(2014)Pennington, Socher, and
  Manning]{pennington-etal-2014-glove}
Jeffrey Pennington, Richard Socher, and Christopher Manning.
\newblock {G}lo{V}e: Global vectors for word representation.
\newblock In \emph{Proceedings of the 2014 Conference on Empirical Methods in
  Natural Language Processing ({EMNLP})}, pages 1532--1543, Doha, Qatar,
  October 2014. Association for Computational Linguistics.
\newblock \doi{10.3115/v1/D14-1162}.
\newblock URL \url{https://aclanthology.org/D14-1162}.

\end{thebibliography}






\newpage
\appendix
\onecolumn

\section{Details of the Proofs}

\subsection{Preliminaries}
\label{sec:prelims}
\begin{definition}[Hypergraph]
Let $H=(V,E, \mX, \mE)$ denote a hypergraph $H$ with a finite vertex set $V=\{v_1, \ldots , v_n\}$, corresponding vertex features with $\mX \in \R^{n \times d}; \, d>0$, a finite hyperedge set $E = \{e_1, \ldots, e_m\}$, where $E \subseteq P^*(V) \setminus \{\emptyset\}$
and $\bigcup\limits_{i=1}^{m}e_i=V$, and $P^*(V)$ denotes the power set on the vertices, the corresponding hyperedge features $\mE \in \R^{m \times d}; \, d > 0$. The hypergraph connections can be conveniently represented as an incidence matrix  $\mathbf{B} \in \{0,1\}^{n \times m}$, where $\mathbf{B}_{ij} =1$ when node $i$ belong to hyperedge $j$ and $\mathbf{B}_{ij}=0$ otherwise.
\end{definition}

\begin{definition}[1-Wesifeiler Lehman Test]
Let $G=(V,E)$ be a graph, with a finite vertex set $V$ and let $s:V \to \Delta$ be a node coloring function with an arbitrary co-domain $\Delta$ and $s(v), v \in V$ denotes the color of a node in the graph.
Correspondingly, we say a labeled graph $(G, s)$ is a graph $G$ with a complete node coloring $s:V \to \Delta$.
The 1-WL Algorithm \citep{doi:10.1137/0209047} can then be described as follows:
let $(G, s)$ be a labeled graph and in each iteration, $t \geq 0,$ the 1-WL computes a node coloring $c_{s}^{(t)}: V \to \Delta,$ which depends on the coloring from the previous iteration.
The coloring of a node in iteration $t>0$ is then computed as
$c_{s}^{(t)}(v)= \text{HASH}\left(\left(c_{s}^{(t-1)}(v),\left\{c_{s}^{(t-1)}(u) \mid u \in N(v)\right\}\right)\right)$
where HASH is bijective map between the vertex set and $\Delta$, and $N(v)$ denotes the 1-hop neighborhood of node $v$ in the graph.
The 1-WL algorithm terminates if the number of colors across two iterations does not change, i.e., the cardinalities of the images of $c_{s}^{(t-1)}$ and $c_{s}^{(t)}$ are equal.
The 1-WL isomorphism test, is an isomorphism test, where the 1-WL algorithm is run in parallel on two graphs $G_1, G_2$ and the two graphs are deemed non-isomorphic if a different number of nodes are colored as $\kappa$ in $\Delta$.
    
\end{definition}

\begin{definition}[Group]
A group is a set $G$ equipped with a binary operation $\cdot: G \times G \to G$ obeying the following axioms:
\begin{itemize}[noitemsep, topsep=5pt, leftmargin=15pt]
    \item for all $g_1, g_2 \in G$, $g_1 \cdot g_2 \in G$ (closure).
    \item for all $g_1, g_2, g_3 \in G$, $g_1 \cdot (g_2 \cdot g_3) = (g_1 \cdot g_2) \cdot g_3$ (associativity).
    \item there is a unique $e \in G$ such that $e \cdot g = g \cdot e = g$ for all $g \in G$ (identity).
    \item for all $g \in G$ there exists $g^{-1} \in G$ such that $g \cdot g^{-1} = g^{-1} \cdot g = e$ (inverse).
\end{itemize}
\end{definition}

\begin{definition}[Permutation Group]
A permutation group (a.k.a. finite symmetric group) $\mathbb{S}_m$ is a discrete group $\gG$ defined over a finite set of size $m$ symbols (w.l.o.g. we shall consider the set $\{1,2,\ldots, m\}$) and consists of all the permutation operations that can be performed on the $m$ symbols.
Since the total number of such permutation operations is $m!$ the order of $\mathbb{S}_m$ is m!.
\end{definition}

\begin{definition}[Product Group]
Given two groups $G$ (with operation $+$) and $H$ (with operation $*$), the direct product group of $G \times H$ is defined as follows:
\begin{itemize}[noitemsep, topsep=5pt, leftmargin=15pt]
    \item The underlying set is the Cartesian product, $G \times H$. That is, the ordered pairs $(g, h)$, where $g \in G$ and $h \in H$
    \item The binary operation on $G \times H$ is defined component-wise:
$(g1, h1) \cdot (g2, h2) = (g1 + g2, h1 * h2)$
\end{itemize}
\end{definition}

\begin{definition}[Group invariant functions]
Let $G$ be a group acting on vector space $V$.
We say that a function $f : V \to \R$ is $G$-invariant
if $f(g \cdot x) = f(x)\;\; \forall x \in V, g \in G$.
\end{definition}

\begin{definition}[Orbit of an element]
For an element $x \in X$ and given a group $G$, the orbit of the element $x$ is given by $\text{Orb}(x) = \{g \cdot x \;\;|\;\; \forall g \in G\}$.
\end{definition}

\subsection{Proofs}
\label{sec:proofs}
In the section, we restate (for the sake of convenience) and prove the statements in the paper.

First, we restate the assumption and the corresponding property and then provide the proof.

\begin{assumption}
\label{assumption-appendix:septable}
For any table $(\gT_{ij})_{i \in [n], j \in [m]}$ (where $i, j$ are indexes of the rows, and columns appropriately), an arbitrary group action $g \in \mathbb{S}_n \times \mathbb{S}_m$ acting appropriately on the rows and columns leaves the target random variables associated with tasks on the entire table unchanged.
\end{assumption}

\begin{property}
A function $\phi:\mathcal{T} \mapsto \mathbf{z} \in R^d$ which satisfies Assumption \ref{assumption-appendix:septable} and defined over tabular data must be invariant to actions from the (direct) product group $\mathbb{S}_n \times \mathbb{S}_m$ acting appropriately on the table i.e. $\phi(g \cdot \gT) = \phi(\gT)\;\; \forall g \in \mathbb{S}_n \times \mathbb{S}_m$.
\end{property}
\begin{proof}
Proof by contradiction.
Let $\mathbf{y} \in \R$ be some target value associated with the entire table -- which is invariant up to permutations of the table given by Assumption \ref{assumption-appendix:septable}. Assume the case where $\phi$ is not an invariant function acting over tables i.e. $\phi(g \cdot \gT) \not= \phi(\gT)\;\; \forall g \in \mathbb{S}_m \times \mathbb{S}_n$.
This means different group actions will result in different output representations-and hence the representation is dependent on the action itself.
When $d=1$ and the value associated with $\mathbf{y}$ is given by $\phi$ acting over the table, equality between the two (for all permutation actions of a given)  is only possible if we can unlearn the group action (or learn the same representation for the entire orbit) which would make it invariant--and hence a contradiction.
\end{proof}

Next, we look at the function $g$ defined via Eqns~(\ref{eqn:start}-\ref{eqn:end}) to see that the representation obtained by a table via message passing is invariant to row/column permutations of the incidence matrix $\mathbf{B}$. Eqn~\ref{eqn:start}, which is the \texttt{Node2hyperedge} block, is a permutation invariant set aggregator which aggregates information from cells in the table to get the vectorial representation of a row/column/table hyperedge. That is, any permutation action drawn from the $\mathbb{S}_{mn}$ group doesn't impact the representation of a given hyperedge as long as the membership is consistent.  

Similarly, the \texttt{Hyperedge2Node} block (Eqn~\ref{eq:2}) is an aggregation function that aggregates information from the hyperedges to the nodes. This ensures that a given node's (in the hypergraph) representation is invariant to permutation actions from the $\mathbb{S}_{m+n+1}$ group as long as the membership is consistent.

Now, since the \texttt{Node2Hyperedge} and \texttt{Hyperedge2Node} blocks act sequentially, its easy to see that (\ref{eqn:start}-\ref{eqn:end}) jointly encode invariance to the product group $\mathbb{S}_{mn} \times \mathbb{S}_{m+n+1}$ when we obtain representations of an incidence matrix as a set (permutation invariant) aggregation of representations of all the $mn$ nodes and the $m+n+1$ hyperedges.

Next, we restate and prove the theorem about the maximal invariance property.


\begin{theorem}
A  continuous function $\phi:\mathcal{T} \mapsto \mathbf{z} \in \mathbb{R}^d$ over tables is maximally invariant if it is modeled as a function $g:\mathbf{B} \mapsto \mathbf{y} \in \mathbb{R}^k$ over the incidence matrix of a hypergraph $\gG$ constructed per Section~\ref{subsec:construction} (Where $g$ is defined via Eqns~(\ref{eqn:start}-\ref{eqn:end})) if $\exists$ a bijective map between the space of tables and incidence matrices (defined over appropriate sizes of tables and incidence matrices). 
That is, $\phi(\mathcal{T}_1) = \phi(\mathcal{T}_2)$ iff $\mathcal{T}_2$ is some combination of row and/or column permutation of $\mathcal{T}_1$ and $g(\mathbf{B_1}) = g(\mathbf{B_2})$ where $\mathbf{B}_1, \mathbf{B}_2$ are the corresponding (hypergraph) incidence matrices of tables $\mathcal{T}_1, \mathcal{T}_2$.
\end{theorem}
\begin{proof}
Proof by construction. 
We will assume the existence of a bijective map between the space of tables and incidence matrices (with appropriate indexing by column names and some canonical row ordering).
Without loss of generality, we assume the number of rows and columns in the table to be $n$ and $m$ respectively, and some canonical ordering of the rows and columns in the table for the sake of convenience.
Hence the number of nodes,  and hyperedges in the hypergraph based on the formulation in Section~\ref{subsec:construction} is $mn$ and $m+n+1$ respectively.
An invariant function over the incidence matrices means that function is invariant to any consistent relabelling of nodes and hyperedges in the hypergraph--hypergraph incidence matrix of size $mn \times (m+n+1)$ and hence (upon node/hyperedge relabeling) invariant to the  $\mathbb{S}_{mn} \times \mathbb{S}_{m+n+1}$ group. It is important to note that any row/column permutation of the table affects the ordering of both the rows and columns of the incidence matrix.

The existence of a bijective map between $\phi$ and $g$ guarantees that we know exactly which elements (cells, rows, columns, table) participate in which nodes/hyperedges and which don't. 
This is evident as a different permutation of the rows/columns would result in different incidence matrices and hence the above information can be deconstructed by obtaining the mappings between all permutations of a table and corresponding incidence matrices.
Now, as this is proof of maximal invariance we have to prove both directions i.e. when the function $\phi$ maps two tables to the same representation-the tables are just row/column permutations of each other, and the $g$ over their corresponding incidence matrices yield identical representations and vice versa. 

{\bf Direction 1}: Assume a continuous function $\phi$ over two tables $\gT_1, \gT_2$ yields the same representation $\phi(\gT_1) = \phi(\gT_2)$.
Now, we need to show that the function $g$ over the corresponding incidence matrices $\mathbf{B}_1, \mathbf{B}_2$ yield the same representation.
Let $\rho$ be the bijective map between the space of tables and incidence matrices.
The guarantee of a bijective map between tables and incidence matrices and the injective (one-to-one) nature of AllSet \citep{chien2022you}
 which is a multiset function that ensures that the hyperedges with the same elements get the same representation and non-identical hyperedges obtain different representations (via Eqns(\ref{eqn:start}-\ref{eqn:end}). Hence, when the outputs of $\phi$ match, the outputs of $g$ also do - and this is possible when the tables are identical up to permutations of each other.
 The universality property of AllSet, combined with the existence of the bijective map ensures the existence of a neural network $g$ for a given continuous function $\phi$.

 {\bf Direction 2}: Assume a function $g$ which maps incidence matrices $\mathbf{B}_1, \mathbf{B}_2$ of two hypergraphs to the same representation. 
Now again, since the bijective map exists, it ensures that we can check if the same nodes and hyperedges map across the two incidence matrices.  This confirms the existence of a bijective map between the family of stars \citep{tyshkevich1996line} of the two hypergraphs (i).
The 1-Weisfeiler Lehman message passing equations of the \texttt{HyperTrans} (Along with the inherited injectivity property of AllSet) module along the star expansion of the hypergraph ensures that the star expansions of the hypergraphs are 1-WL isomorphic \citep{xu2018powerful} (ii). Combining (i) and (ii) together via Theorem 2 from \citep{srinivasan2021learning}, we get that the hypergraphs are isomorphic - and therefore yielding the tables to be identical and $\phi$ getting the same representation due to Property~\ref{pro:1}.
It is easy to see that when the $g$ yields different representations, the tables are not identical, and hence $\phi$ yields a different representation as well.
\end{proof}



\section{Further Analysis of \textsc{HyTrel}}
\label{appdx:further_analysis}


\subsection{Effect of Input Table Size}
\label{appdx:tab_size}

As mentioned in Section~\ref{sec:exp}, we have truncated the large tables and only keep a maximum of 30 rows and 20 columns in pretraining for efficiency purposes. However, the truncation does not affect \textsc{HyTrel}'s ability to deal with large tables in downstream tasks. Our model can deal with the arbitrary size of tables. This is different from the BERT-based TaLMs (e.g., TaBERT, TURL) that have positional encodings, and the pretrained checkpoints limit the positions within a length of 512 for an input sequence. With the \textsc{HyTrel}, we can always build large hypergraphs with large tables by adding more nodes under each hyperedge, without sacrificing information loss.

To further demonstrate this, we do experiments on the table type detection (TTD) dataset, as it has relatively large tables (an average number of 157 rows). Below in Table~\ref{tab:input} are the experimental results of using different maximum row/column limitations with  \textsc{HyTrel} (ELECTRA), and we fine-tune the dataset on an NVIDIA A10 Tensor Core GPU (24GB).

\begin{table}[h]
\small
\centering
	\begin{tabular}{@{}ccccc@{}}
		\toprule
        \toprule
		\begin{tabular}[c]{@{}l@{}}Size Limitations  \\(\#rows, \#columns) \end{tabular} & \begin{tabular}[c]{@{}l@{}} Dev Acc (\%) \end{tabular} & \begin{tabular}[c]{@{}l@{}} Test Acc (\%) \end{tabular} & \begin{tabular}[c]{@{}l@{}} Training Time \\ (min/epoch)  \end{tabular} & \begin{tabular}[c]{@{}l@{}} GPU Memory Consumption \\ (memory used/batch size, GB/sample) \end{tabular}  \\ \midrule
            (3, 2)    & 78.47 &	78.00	& 3	&  0.14 \\
            (15, 10)    & 95.36 &	95.96	& 9	& 0.48 \\
		(30, 20)    & 96.23 &	95.81	& 14	& 0.58 \\
		(60, 40)   & 96.38 &	95.80	& 23 &	1.41  \\ 
            (120, 80)  & 96.02 &	95.71 &	51 &	5.13 \\ 
            (240, 160)   & 96.06 &	95.67	& 90 &	16.00 \\ 
        \bottomrule
        \bottomrule
	\end{tabular}
\caption{Effect of Input Table Size}
\label{tab:input}
\end{table}

We can see that \textsc{HyTrel} has no difficulty dealing with the arbitrary size of input tables. However, we also observe that the prediction performances are almost the same when we enlarge the table size limitation to a certain number, but the time and hardware resource consumption will increase dramatically. So we recommend using downsampling to deal with large tables, and it will be much more efficient without hurting the performance. 


\subsection{Effect of Excessive Invariance}
\label{appdx:invariance}

As we have emphasized, the invariant permutation property of a table restricts the permutation of rows and columns to be independent. Arbitrary permutations of rows or columns can lead to an excessive invariance problem, ultimately compromising the table's semantics. Take Figure ~\ref{fig:inv} as an example, permuting the rows and columns independently will result in semantically the same table (Upper), but arbitrarily permuting will break the semantics of the table and cause excessive invariance (Below).

We use the TaBERT model to demonstrate the effect brought by the excessive invariance. As we have introduced before, the TaBERT models use a linearization approach with positional encoding to model a table. So they cannot preserve the table invariance property. However, we cannot achieve this by simply removing the positions, as it will cause excessive invariance and it is equivalent to shuffling all the cells in a table. As we emphasize, we only assume independent row/column permutation invariance. We empirically show the effect of disabling the positional encodings for the TaBERT model. As shown in Table~\ref{tab:inv}, we evaluate the model on the Table Type Detection dataset. It turns out that the excessive invariance will cause a significant performance decrease.

\begin{figure}[h]
\begin{center}
\includegraphics[width=320pt]{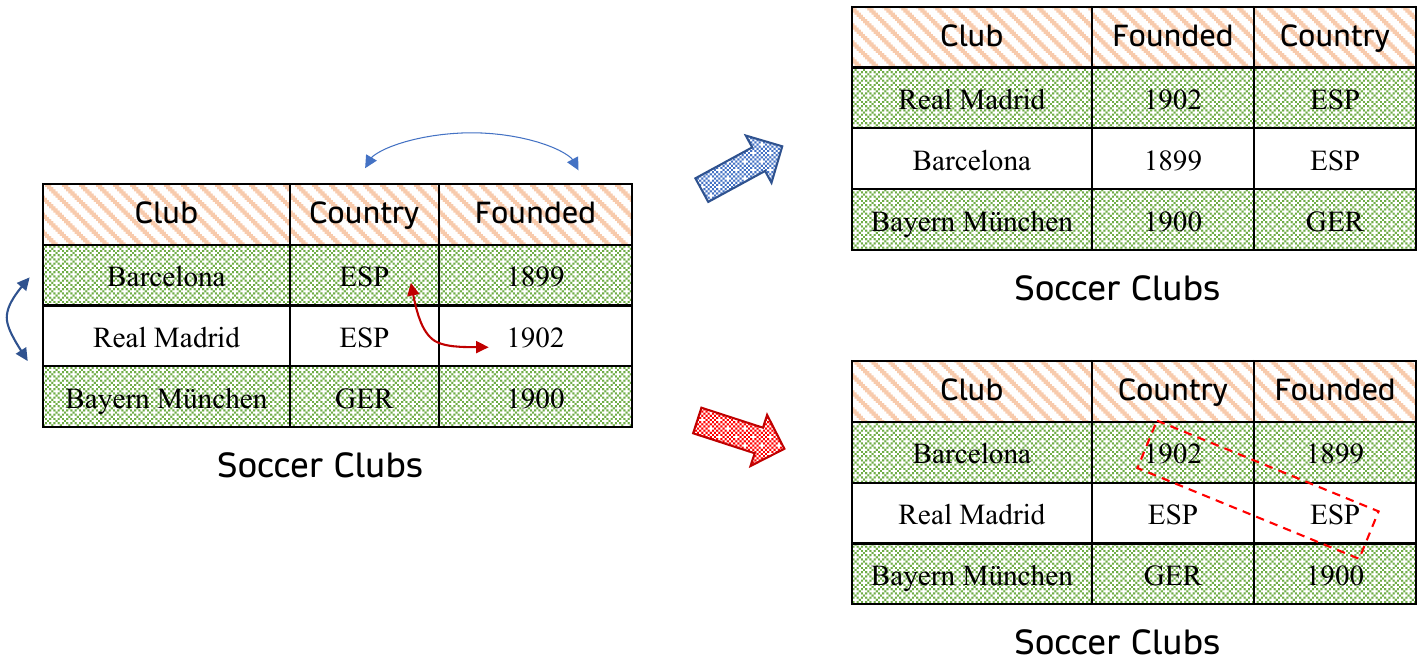}
\caption{An example of table permutations.  }
\label{fig:inv}
\end{center}
\end{figure}

\begin{table*}[h]
\centering

	\begin{tabular}{@{}ccc@{}}
		\toprule
        \toprule
		\begin{tabular}[c]{@{}l@{}} Models \end{tabular} & \begin{tabular}[c]{@{}l@{}} Dev Acc (\%) \end{tabular} & \begin{tabular}[c]{@{}l@{}} Test Acc (\%) \end{tabular}   \\ \midrule
            TaBERT (K=1), w/ positions    & 94.11	& 93.44	 \\
            TaBERT (K=1), w/o positions	    & 88.62	& 88.91	 \\
		TaBERT (K=3), w/ positions   & 95.82	& 95.22	 \\
		TaBERT (K=3), w/o positions  & 92.80 &	92.44	  \\ 
        \bottomrule
        \bottomrule
        \end{tabular}
\caption{Excessive invariance causes performance slumping}
\label{tab:inv}
\end{table*}







\subsection{Inference Complexity of the Model}
\label{appdx:complex}

{\bf Theoretical Analysis}: Given a table with $m$ rows and $n$ columns, let’s assume each cell only contains one token, and we use dot-product and single head in the attention operation. The inference complexity comparison of one layer attention is:

\begin{table*}[h]
\centering
	\begin{tabular}{@{}ccc@{}}
		\toprule
        \toprule
		\begin{tabular}[c]{@{}l@{}} Complexity \end{tabular} & \begin{tabular}[c]{@{}l@{}} BERT-based Methods \end{tabular} & \begin{tabular}[c]{@{}l@{}} \textsc{HyTrel} \end{tabular}   \\ \midrule
            w/o linear transformation    & $O((mn)^2\cdot d)$	& $O(mnd)$ \\
            w/ linear transformation	    & $O((mn)^2\cdot d + (mn)\cdot d^2)$& $O(mnd+(mn)\cdot d^2)$	 \\
        \bottomrule
        \bottomrule
        \end{tabular}
\caption{Inference Complexity}
\label{tab:inv}
\end{table*}

For the BERT-based methods, the input sequence length is $mn$, $d$ is the hidden dimension. We can see that our HyTrel has less complexity, and the difference comes from the dot-product for attention scores. The query in the set transformer of HyTrel is a learnable vector with $d$ dimension, it reduces computation time of self-attention from quadratic to linear in the number of elements in the set~\citep{lee2019set}, with a complexity of $O(mnd)$. In the self-attention, the query is a matrix from the whole input sequence with $(mn) \times d$ dimensions, which has a complexity of $O((mn)^2\cdot d)$ in the attention score calculation.

{\bf Empirical Analysis}: We include comparison of our approach with the TaBERT baseline about the inference time, including data processing (graph building time in HyTrel). As observable from the table, the total inference time of our model is lesser compared to TaBERT, and the time cost for graph building is not a bottleneck.

\begin{table}[h]
\centering
\begin{tabular}
{l r r r r}
\toprule
\toprule
{Datasets}                 & { \textbf{CTA}} & { \textbf{CPA}} & { \textbf{TTD }} & {\textbf{TS}} \\
\midrule
{Total Samples}                              & {4844}                           & {1560}                           & {4500}                           & {1391}                          \\
{batch size}                              & {8}                           & {8}                           & {8}                           & {8}                          \\
\midrule
{TaBERT (K=1) - data preparing}           & {13}                          & {5}                           & {51}                          & {13}                         \\
{TaBERT (K=1) - inference}                & {27}                          & {5}                           & {25}                          & {28}                         \\
{\textbf{TaBERT (K=1) - total inference}} & { \textbf{40}}                 & { \textbf{10}}                 & { \textbf{76}}                 & { \textbf{41}}                \\
\midrule
{TaBERT (K=3) - data preparing}           & {14}                          & {5}                           & {52}                          & {13}                         \\
{TaBERT (K=3) - inference}                & {85}                          & {16}                          & {79}                          & {80}                         \\
{\textbf{TaBERT (K=3) - total inference}} & { \textbf{99}}                 & { \textbf{21}}                 & { \textbf{131}}                & { \textbf{93}}                \\
\midrule
{HyTrel - graph building}                 & {5}                           & {1}                           & {16}                          & {5}                          \\
{HyTrel - inference}                      & {16}                          & {6}                           & {17}                          & {13}                         \\
{\textbf{HyTrel - total inference}}       & { \textbf{21}}                 & { \textbf{7}}                  & { \textbf{33}}                 & { \textbf{18}}       \\
\bottomrule
\bottomrule
\end{tabular}
\vspace{5pt}
\caption{Inference Time Comparision (seconds)}
\end{table}

\begin{itemize}[noitemsep, topsep=5pt, leftmargin=*]
  \item We keeps a maximal of 3 rows with the HyTrel model for fair comparison with the BERT (K=3) models.
  \item The data preprocessing of TaBERT is to format the input tables (.json) into tensors, and the graph building of HyTrel is to format the input tables (.json) into feature tensors and the incidence matrix in hypergraphs.
  \item All experiments are conducted on a single A10 GPU, and the inference batch sizes are all chosen to be 8 for all models and all dataset;
  \item We use the validation set of CTA, CPA and TTD for experiments. For TS with a small number of tables that is tested with five fold cross-validation in the paper, here we use the whole dataset for experiments.
\end{itemize}

\section{Details of the Experiments}
\label{appdx:exp}

\subsection{Pretraining}
\label{appdx:pretrain}

\subsubsection{Model Size}

The parameter size of the \textsc{HyTrel} and the baseline TaBERT are as followed in Table~\ref{tab:size}. \textsc{HyTrel} has more parameters mainly because of the \texttt{Hyperedge Fusion} block (Figure ~\ref{fig:framework}) that is absent from the TaBERT model.

\begin{table*}[h]
\centering
	\begin{tabular}{@{}ccc@{}}
		\toprule
        \toprule
	\begin{tabular}[c]{@{}l@{}} Models \end{tabular} & \begin{tabular}[c]{@{}l@{}} \# Parameters (million) \end{tabular}   \\ \midrule
            TaBERT (K=1)   & 110.6		 \\
            TaBERT (K=3)    & 131.9		 \\
		\textsc{HyTrel}  & 179.4		 \\
        \bottomrule
        \bottomrule
        \end{tabular}
\caption{Model Size}
\label{tab:size}
\end{table*}

\subsubsection{Hyperparameters}
The hyperparameter choices and resource consumption details of the pretraining are presented in Table~\ref{tab:param1} and Table~\ref{tab:cost}.


\begin{table}[h]
\small
\centering
	\begin{tabular}{@{}ccccccc@{}}
		\toprule
        \toprule
		\begin{tabular}[c]{@{}l@{}}Models \end{tabular} & \begin{tabular}[c]{@{}l@{}} Batch \\ Size \end{tabular} & \begin{tabular}[c]{@{}l@{}} Learning \\ Rate \end{tabular} & \begin{tabular}[c]{@{}l@{}} Warm-up \\ Ratio  \end{tabular} & \begin{tabular}[c]{@{}l@{}} Max \\ Epochs \end{tabular} & \begin{tabular}[c]{@{}l@{}} Weight \\ Decay \end{tabular} & \begin{tabular}[c]{@{}l@{}} Optimizer \end{tabular}  \\ \midrule
		\textsc{HyTrel}  \textit{w/} ELECTRA    & 8192  &  5e$^{-4}$  & 0.05   & 5 & 0.02 & Adam  \\
		\textsc{HyTrel}  \textit{w/} Contrastive  & 4096  & 1e$^{-4}$  & 0.05   & 5  & 0.02 & Adam  \\ 
        \bottomrule
        \bottomrule
	\end{tabular}
\caption{Pretraining Hyperparameters}
\label{tab:param1}
\end{table}

\begin{table}[h]
\small
\setlength{\tabcolsep}{5pt}
\centering
	\begin{tabular}{@{}ccccc@{}}
		\toprule
        \toprule
		\begin{tabular}[c]{@{}l@{}} Models \end{tabular} & \begin{tabular}[c]{@{}l@{}} GPUs \end{tabular} & \begin{tabular}[c]{@{}l@{}} Accelerator \end{tabular} & \begin{tabular}[c]{@{}l@{}} Precision   \end{tabular} & \begin{tabular}[c]{@{}l@{}}  Training  Time \end{tabular}  \\ \midrule
		\textsc{HyTrel}  \textit{w/} ELECTRA    & 16 $\times$ NVIDIA A100  &  DeepSpeed (ZeRO Stage 1)  & bf16   &   6h / epoch \\
		\textsc{HyTrel}  \textit{w/} Contrastive  & 16 $\times$ NVIDIA A100  & DeepSpeed (ZeRO Stage 1)  & bf16   & 4h / epoch  \\ 
        \bottomrule
        \bottomrule
	\end{tabular}
\caption{Pretraining Cost}
\label{tab:cost}
\end{table}

\subsection{Fine-tuning}
\label{appdx:fine-tuning}


\subsubsection{Details of the Datasets}


 {\bf TURL-CTA} is a multi-label classification task (this means for one column, there can be multiple labels annotated) that consists of 255 semantic types. In total, it has 628,254 columns from 397,098 tables for training, and 13,025 (13,391) columns from 4,764 (4,844) tables for testing (validation).

 {\bf TURL-CPA} dataset is also a multi-label classification task in which each pair of columns may belong to multiple relations. It consists of a total number of 121 relations with 62,954 column pairs from 52,943 tables training sample, and 2,072 (2,175) column pairs from 1,467 (1,560) tables for testing (validation). Similar to the CTA task, we use the pairwise column representations from \textsc{HyTrel} with their corresponding hyperedge representations for fine-tuning.    

{\bf WDC Schema.org Table Corpus} consists of 4.2 million relational tables covering 43 schema.org entity classes. We sample tables from the top 10 classes and construct a table type detection dataset that contains 3,6000 training tables and 4,500 (4,500) tables for validation (testing).

 {\bf PMC} is a table corpus that is formed from PubMed Central (PMC) Open Access subset, and this collection is about biomedicine and life sciences. In total, PMC contains 1,391 table pairs, where 542 pairs are similar and 849 pairs are dissimilar. For each pair of tables, we use their table representations together with column representations from \textsc{HyTrel} for fine-tuning. As for evaluation, we follow previous work~\citep{tabsim_inproceedings, trabelsi22www} and report the macro-average of the 5-fold cross-validation performance.

\subsubsection{Hyperparameters}

The hyperparameter choices for fine-tuning both the TaBERT and the \textsc{HyTrel} models are as followed. For the primary hyperparameters \texttt{batch size} and \texttt{learning rate}, we choose the \texttt{batch size} to be as large as possible to fit the GPU memory constraint, and then use grid search to find the best \texttt{learning rate}. Other hyperparameters are kept the same as in previous work.

\begin{table}[h]
\centering
	\begin{tabular}{@{}ccccccc@{}}
		\toprule
        \toprule
		\begin{tabular}[c]{@{}l@{}}Tasks \end{tabular} & \begin{tabular}[c]{@{}l@{}} Batch \\ Size \end{tabular} & \begin{tabular}[c]{@{}l@{}} Learning \\ Rate \end{tabular} & \begin{tabular}[c]{@{}l@{}} Warm-up \\ Ratio  \end{tabular} & \begin{tabular}[c]{@{}l@{}} Max \\ Epochs \end{tabular} & \begin{tabular}[c]{@{}l@{}} Weight \\ Decay \end{tabular} & \begin{tabular}[c]{@{}l@{}} Optimizer \end{tabular}   \\ \midrule
		Column Type Annotation    & 256  &  1e$^{-3}$  & 0.05   & 50 & 0.02 & Adam \\
            Column Property Annotation    & 256  &  1e$^{-3}$  & 0.05   & 30 & 0.02 & Adam  \\
            Table Type Detection    & 32  &  5e$^{-5}$  & 0.05   & 10 & 0.02 & Adam \\
            Table Similarity Prediction    & 16  &  5e$^{-5}$  & 0.05   & 10 & 0.02 & Adam   \\
        \bottomrule
        \bottomrule
	\end{tabular}
\caption{Hyperparameters for Fine-tuning \textsc{HyTrel}. We use the same settings for all variants of \textsc{HyTrel} (no pretraining, pretraining with ELECTRA and Contrastive objectives)}
\label{tab:param2}
\end{table}

\begin{table}[h]
\small
\setlength{\tabcolsep}{3pt}
\centering
\begin{tabular}{@{}cccccc@{}}
\toprule
\toprule
\begin{tabular}[c]{@{}l@{}}Tasks \end{tabular} & \begin{tabular}[c]{@{}l@{}} Batch \\ Size \end{tabular} & \begin{tabular}[c]{@{}l@{}} Learning \\ Rate \end{tabular} & \begin{tabular}[c]{@{}l@{}} Warm-up \\ Ratio  \end{tabular} & \begin{tabular}[c]{@{}l@{}} Max \\ Epochs \end{tabular} & \begin{tabular}[c]{@{}l@{}} Optimizer \end{tabular}   \\ \midrule
Column Type Annotation    & 64 (K=3), 128 (K=1)  &  5e$^{-5}$  & 0.1   & 10  & Adam \\
    Column Property Annotation    & 48 (K=3), 128 (K=1)   &  1e$^{-5}$ (K=3), 2e$^{-5}$ (K=1)  & 0.1   & 10  & Adam  \\
    Table Type Detection    & 64  &  5e$^{-5}$  & 0.1   & 10 &  Adam \\
    Table Similarity Prediction    & 8  &  5e$^{-5}$  & 0.1   & 10 &  Adam   \\
\bottomrule
\bottomrule
\end{tabular}
\caption{Hyperparameters for Fine-tuning TaBERT.}
\label{tab:param3}
\end{table}

\subsubsection{Details of the Baselines}
\label{appdx:baseline}

\noindent {\bf TaBERT}~\citep{yin20acl} is a representative TaLM that jointly learns representations for sentences and tables. It flattens the tables into sequences and pretrains the model from BERT checkpoints. Since we use the same source and prepossessing for the pretraining tables, we use it as our main baseline. TaBERT is not yet evaluated on the table-only tasks, we fine-tune the \textit{base} version with both \textit{K=1} and \textit{K=3} row number settings on all the tasks. We also evaluate the TaBERT model that is randomly initialized and see how much it relies on the pretraining.

\noindent {\bf TURL}~\citep{deng2020turl} is another representative TaLMs that focuses on web tables with cell values as entities connected to the knowledge base. It also flattens the tables into sequences and pretrains from TinyBERT~\citep{jiao-etal-2020-tinybert} checkpoints. It introduces a vision matrix to incorporate table structure into the representations. We copy the reported results for CTA and CPA tasks. For a fair comparison, we use the fine-tuning settings with metadata (captions and headers) but without external pretrained entity embeddings. 

\noindent {\bf  Doduo}~\citep{10.1145/3514221.3517906} is a state-of-the-art column annotation system that fine-tunes the BERT and uses table serialization to incorporate table content. We copy the reported results for the CTA and CPA tasks.

For each specific task and dataset, we also report the state-of-the-art baseline results from previous work. For the TURL-CTA dataset, we use the reported results of Sherlock~\citep{Hulsebos:2019:SDL:3292500.3330993}. For the TURL-CPA dataset, we use the reported results of fine-tuning the BERT$_{base}$~\citep{devlin-etal-2019-bert} model. For the PMC dataset, we use the reported results from two baselines: TFIDF+Glove+MLP, which uses TFIDF and Glove~\citep{pennington-etal-2014-glove} features of a table with a fully connected MLP model, and TabSim~\citep{tabsim_inproceedings} that uses a siamese Neural Network to encode a pair of tables for prediction.   

\end{document}